\pdfoutput=1
\documentclass{article} %
\usepackage{iclr2027_conference,times}

\usepackage[utf8]{inputenc}
\usepackage[T1]{fontenc}
\usepackage[colorlinks=true,linkcolor=blue!45!black,citecolor=blue!45!black,urlcolor=blue!45!black]{hyperref}
\usepackage{url}
\usepackage{booktabs}
\usepackage{longtable}
\usepackage{amsfonts}
\usepackage{amsmath, amssymb}
\usepackage{amsthm}
\usepackage{nicefrac}
\usepackage{microtype}
\usepackage{xcolor}
\usepackage{graphicx}
\usepackage{wrapfig}
\usepackage{caption}
\usepackage{enumitem}
\usepackage{algorithm}
\usepackage[noend]{algpseudocode}
\usepackage{tcolorbox}
\tcbuselibrary{breakable}

\theoremstyle{plain}
\newtheorem{theorem}{Theorem}
\newtheorem{lemma}[theorem]{Lemma}

\newtheorem{corollary}[theorem]{Corollary}
\newtheorem{assumption}{Assumption}
\theoremstyle{definition}

\newtheorem{remark}{Remark}

\iclrfinalcopy

\title{Constraint Tree Exploration\\ for Learning from Language Feedback}

\author{
Shaoang Li$^{1}$ \qquad
Daniel R. Jiang$^{2}$ \qquad
Jian Li$^{1}$ \\
$^{1}$Stony Brook University
\qquad
$^{2}$Meta Platforms \\
\texttt{\{shaoang.li,jian.li.3\}@stonybrook.edu}
\qquad
\texttt{danielrjiang@gmail.com}
}

\begin{document}

\maketitle
\lhead{Preprint}

\begin{abstract}

Natural-language feedback in interactive learning often explains \emph{why} an action failed by pointing to violated requirements. Misinterpreting this feedback can lead an agent to rule out valid solutions. We study this setting by modeling user intent as latent constraints over an action space and formulating learning from language feedback as pure exploration over feasible regions. We introduce \textsc{TRACE}, an algorithm that organizes candidate constraints in a tree and tests each proposed refinement by generating actions that satisfy it. \textsc{TRACE} commits to the refinement only if the resulting feedback does not contradict it over repeated tests. We distinguish two ways of using the same feedback: (i) \emph{falsification}, which detects contradictions to the constraint set currently being tested, and (ii) \emph{identification}, which may additionally name a violated constraint. We prove high-probability coverage bounds with dependence on the candidate class size $H$ for \textsc{TRACE-Falsification}. With reliable identification, \textsc{TRACE-Identification} can replace this dependence by $K/p_{\mathrm{ext}}$, where $K$ is the number of latent constraints and $p_{\mathrm{ext}}$ lower-bounds the probability of extracting a missing true constraint from informative feedback. We evaluate \textsc{TRACE} across six language-feedback tasks. On RecMovie, \textsc{TRACE-Identification} achieves 73\% and 86\% final-output success under caps of 20 and 60 evaluated outputs, compared with at most 42\% and 48\% for the evaluated prompting baselines given the same feedback and output caps. Controlled identity-corruption experiments further show greater robustness than direct accumulation when the falsification detector remains reliable.
\end{abstract}

\section{Introduction}
\label{sec:introduction}

Interactive agents built on large language models (LLMs) can receive feedback expressed in natural language rather than as a numerical reward~\citep{cheng2023llf,shinn2023reflexion}. 
A user may reject a movie recommendation by saying ``I wanted something from the 90s'', revise a travel plan by saying ``this exceeds my budget'', or correct a generated poem by saying ``the second line has too many syllables''. 
Such feedback can explain \emph{why} an action failed by identifying an unmet requirement. Because this requirement also applies to other candidate actions, the feedback can help narrow the search for acceptable alternatives. The challenge is to infer useful constraints from incomplete or ambiguous feedback and determine when the available evidence justifies enforcing them.

Existing approaches have explored how natural-language feedback can be incorporated into LLM agents. %
Feedback-driven methods, including \textsc{Reflexion}~\citep{shinn2023reflexion}, \textsc{Self-Refine}~\citep{madaan2023selfrefine}, \textsc{TextGrad}~\citep{yuksekgonul2025optimizing}, and language-agent tree search (\textsc{LATS})~\citep{zhou2024language}, use feedback to guide generation and search through revised prompts, reflective memories, textual gradients, or value estimates. %
Recent theoretical work provides complementary abstractions, including regret bounds under hindsight instructions~\citep{misra2024provable} and consistency-based elimination over global hypotheses~\citep{xu2025helix}. 
We focus on pure exploration over feasible regions: the learner discovers, tests, and accumulates latent constraints to progressively restrict the action space. This formulation makes explicit both the information needed to propose a constraint and the evidence needed to commit to it.

\begin{figure}[t]
    \centering
    \includegraphics[width=0.95\linewidth]{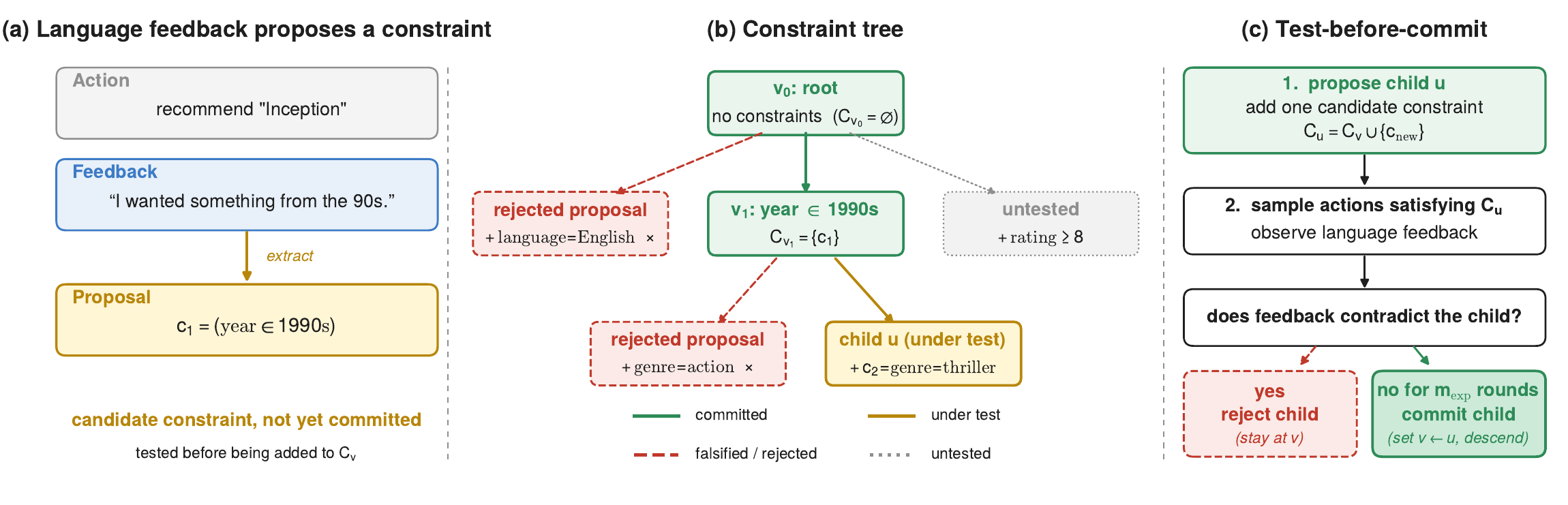}
    \vspace{-0.1in}
    \caption{Constraint-tree exploration in \textsc{TRACE}. At a committed node $v$ with constraint set $C_v$, a candidate constraint $c_{\mathrm{new}}$ proposes a child $u$ with $C_u=C_v\cup\{c_{\mathrm{new}}\}$. The child is tested by sampling actions satisfying $C_u$: feedback contradicting $C_u$ rejects the child, while survival for a fixed number $m_{\mathrm{exp}}$ of tests leads to commitment. Extracted identities guide proposals; falsification tests determine commitment.}
    \label{fig:constraint-tree}
    \vspace{-0.15in}
\end{figure}

We consider episodes with fixed latent requirements represented in a finite candidate class. We distinguish two oracle-access regimes induced by the same natural-language feedback.
In the \emph{falsification setting}, the learner has access to a violation oracle that tests whether feedback contradicts the constraint set currently being tested. In the \emph{identification setting}, the learner additionally has access to an extractor that may name a violated atomic constraint. 
These two regimes specify what information the learner can obtain from the same raw feedback.
Identification access can accelerate exploration by revealing which constraint to test next. But an extracted identity is only a proposal, not a fact: %
committing an incorrectly extracted constraint can permanently shrink the feasible region in the wrong direction. Without identities, however, discovering useful constraints may require blind search over a large candidate class.
The challenge is therefore to benefit from named constraints without trusting them too early.  This leads to the central question of the paper:
\vspace{-0.05in}
\begin{quote}
\emph{How can an agent exploit named constraints from language feedback without committing to them prematurely, and when does identification access reduce the sample complexity of exploration?}
\end{quote}
\vspace{-0.05in}
We answer this question with the \emph{Test-and-Refine Algorithm for Constraint Exploration} (\textsc{TRACE}). The search state in \textsc{TRACE} is the committed constraint set, which indexes a surviving feasible region. Candidate refinements form children in a constraint tree, as illustrated in Figure~\ref{fig:constraint-tree}. Rather than committing to a proposed constraint immediately, \textsc{TRACE} performs a multi-round test: it samples actions from the child region and rejects the child if feedback contradicts the tested constraint set. If the child survives repeated tests, the algorithm commits to it and descends the tree. Thus, extracted identities are treated as \emph{proposals}, not facts. The tree %
separates proposal from commitment and organizes partially accumulated constraints. The same test-before-commit primitive instantiates the two oracle-access regimes formalized in Section~\ref{sec:tree}: \textsc{TRACE-Falsification} (\textsc{TRACE-F}; Assumption~\ref{ass:falsification-access} only), which uses the violation oracle for falsification, and \textsc{TRACE-Identification} (\textsc{TRACE-I}; Assumptions~\ref{ass:falsification-access} and~\ref{ass:identification-access}), which additionally prioritizes constraints extracted from feedback.

Our analysis separates the cost of blind falsification from the value of usable identities. In the falsification setting, \textsc{TRACE-F} discovers true constraints by proposing candidates from a finite class of size $H$, leading to sample complexity %
$\widetilde{\mathcal{O}}(H)$ up to the per-test cost of falsification (Theorem~\ref{thm:main-noisy}). %
In the identification setting, %
under sound extraction and an explicit discovery-progress condition, 
\textsc{TRACE-I} replaces this dependence by the effective discovery cost $\widetilde{\mathcal{O}}(K/p_{\mathrm{ext}})$, where $K$ is the number of latent constraints and $p_{\mathrm{ext}}$ lower-bounds the probability of extracting a missing true constraint from informative feedback (Theorem~\ref{thm:identity}). %
These bounds quantify how usable identities can reduce the number of tested refinements relative to uniform blind search.

Empirically, we evaluate \textsc{TRACE} across five controlled language-feedback environments spanning six task instances: recommendation, two entity-deduction benchmarks, \textsc{Mastermind}, and constrained poem generation. %
On RecMovie, \textsc{TRACE-I} achieves 73\% and 86\% final-output success under common caps of 20 and 60 evaluated outputs, compared with at most 42\% and 48\% for the evaluated prompting baselines. Controlled \textsc{Mastermind} sweeps show that identification-guided proposals require fewer tested refinements as the candidate class grows. False-identity experiments further show that \textsc{TRACE-I} remains more robust than direct accumulation when extracted identities are corrupted and the falsification detector remains reliable.

In summary, we make three contributions. \textit{First,} we formulate learning from language feedback as pure exploration over latent constraints, where feedback provides local evidence about feasible regions rather than scalar rewards. \textit{Second,} we introduce \textsc{TRACE}, a parser-agnostic test-before-commit framework that treats extracted identities as proposals rather than facts. %
\textit{Third,} we derive high-probability coverage bounds that quantify the benefit of identification access, and we evaluate proposal efficiency, interaction budgets, and robustness to identity corruption.

\section{Problem Formulation and Constraint Tree}
\label{sec:formulation}
We formulate learning from language feedback as a pure-exploration problem over latent constraints. The learner never observes numerical utility values; instead, it receives natural-language feedback that may reveal violated requirements. We then introduce the constraint tree, which organizes partially accumulated constraints as feasible regions to be tested and refined.

\subsection{Language-feedback pure exploration}
\label{sec:setup}

\textbf{Action space and latent utility.}
Let $\mathcal{X}$ be a finite action space and let $U:\mathcal{X}\to[0,1]$ be an unknown latent utility function. The interaction proceeds for $T$ rounds. In each round $t\in[T]$, an algorithm $\mathcal{A}$ selects an action $x_t\in\mathcal{X}$ based on the past observations $(x_1,f_1,\ldots,x_{t-1},f_{t-1})$, and receives natural-language feedback $f_t\in\mathcal{F}$. The algorithm never observes $U(x_t)$ directly.

\textbf{Language feedback as constraint evidence.}
For each action $x\in\mathcal{X}$, feedback is drawn from a conditional distribution
$P_x$ over $\mathcal{F}$. Conditioned on $x_t$, the feedback $f_t$ is independent of the past. %
We study the regime in which feedback may identify requirements that the action violates, provide qualitative corrections, or confirm that the action is close to satisfying the user's intent. 
For example, a user may say ``the recommendation is not from the 90s'' or ``the answer is good, but it misses the required attribute''. We treat such feedback as a \emph{constraint signal}: %
progress comes from discovering and testing latent requirements, rather than from comparing observed numerical rewards.

\textbf{Canonical constraint class.}
Let $\mathcal{H}$ be a finite class of canonical candidate constraints, with $H:=|\mathcal{H}|$, within which equality and set membership are well-defined. Each constraint $c\in\mathcal{H}$ defines a region $\mathrm{Region}(c)\subseteq\mathcal{X}$, and we write $\mathrm{Sat}(x,c):=\mathbf{1}\{x\in\mathrm{Region}(c)\}$ for the satisfaction indicator. For a constraint set $C\subseteq\mathcal{H}$, define $\mathrm{Region}(C):=\bigcap_{c\in C}\mathrm{Region}(c)$, with $\mathrm{Region}(\emptyset):=\mathcal{X}$. The latent constraints below are assumed to lie in $\mathcal{H}$, i.e., the algorithm searches over a class that contains them.

\textbf{Latent constraints.}
The user's hard requirements are encoded by an unknown subset $C^*:=\{c_1^*,\ldots,c_K^*\}\subseteq\mathcal{H}$ of {$K:=|C^*|$} latent constraints, with feasible region $\mathrm{Region}(C^*)=\bigcap_{k=1}^K \mathrm{Region}(c_k^*)$. We assume that these constraints are necessary for optimality, i.e., every globally optimal action satisfies all of $C^*$: $\mathcal{X}^*:=\arg\max_{x\in\mathcal{X}} U(x)\neq\emptyset$ and $\mathcal{X}^*\subseteq\mathrm{Region}(C^*)$. Let $x^*\in\mathcal{X}^*$ and $U^*:=U(x^*)$.
Apart from this compatibility condition, the utility function $U$ is arbitrary: actions outside $\mathrm{Region}(C^*)$ may still have positive utility, and $U$ may take any values in $[0,1]$ rather than being restricted to $\{0,1\}$. Thus, the constraints define the feasible region containing all global optima, while the utility function determines how costly it is to violate or miss them.

\textbf{Coverage objective and regret consequence.}
After $T$ rounds, the algorithm outputs a single action $\hat{x}$. The %
primary objective in this paper is \textit{coverage:} %
reach a nonempty region contained in $\mathrm{Region}(C^*)$ and sample the output from that region.
The utility function is not observed during interaction; it is used only to express the consequence of coverage for simple regret,
\begin{align}
\mathbb{E}[r_T]
:=
\mathbb{E}\big[U^*-U(\hat{x})\big],
\end{align}
where the expectation is over the algorithm's randomness and the feedback channel. %
As formalized below, if the algorithm reaches a region contained in $\mathrm{Region}(C^*)$, then the remaining regret is controlled by the variation of $U$ within the latent feasible region.

\subsection{Constraint tree representation}
\label{sec:tree}

The constraint tree organizes partial information about the latent feasible region. Each node corresponds to a set of constraints accumulated so far and therefore to the region of actions satisfying them. Moving deeper in the tree adds constraints and shrinks the feasible region.

\textbf{Nodes as constraint sets.}
A constraint tree $\mathcal{T}$ is a rooted tree of proposal paths. %
Each node $v$ is associated with the accumulated constraint set $C_v\subseteq\mathcal{H}$ and feasible region
\begin{align}
\mathcal{X}_v:=\mathcal{X}(C_v):=\{x\in\mathcal{X}:\mathrm{Sat}(x,c)=1, \forall c\in C_v\}
=\mathrm{Region}(C_v).
\end{align}
The root $v_0$ has $C_{v_0}=\emptyset$ and $\mathcal{X}_{v_0}=\mathcal{X}$. Each child $u$ of $v$ adds one new constraint, so $C_u=C_v\cup\{c\}$ for some $c\in\mathcal{H}\setminus C_v$. The depth of $v$ is $h(v):=|C_v|$. Different proposal orders may induce the same accumulated set and feasible region; we keep them as distinct tree nodes because they correspond to different search histories.

\textbf{Viable and spurious nodes.}
For any node $v$ with a nonempty feasible region, define
\begin{align}
U(v):=\max_{x\in\mathcal{X}_v}U(x), \qquad \Delta(v):=U^*-U(v).
\end{align}
We call $v$ \emph{viable} if $\Delta(v)=0$, i.e., $\mathcal{X}_v$ contains a globally optimal action, and \emph{spurious} if $\Delta(v)>0$. The condition $C_v\subseteq C^*$ is sufficient for viability, but not necessary: a redundant or stronger-looking constraint outside $C^*$ may still preserve at least one global optimum. The algorithm uses feedback to reject spurious branches while preserving viable ones.

\textbf{Within-region spread.}
Even a viable region may contain actions of different utility. We define the within-node spread $\mathrm{Spread}(v):=U(v)-\inf_{x\in\mathcal{X}_v}U(x),$ and, for the latent feasible region,
\begin{align}
\mathrm{Spread}(C^*):=\sup_{x\in\mathrm{Region}(C^*)}U(x)-\inf_{x\in\mathrm{Region}(C^*)}U(x).
\end{align}
If $\mathcal{X}_v\subseteq\mathrm{Region}(C^*)$, then any $\hat{x}\in\mathcal{X}_v$ satisfies all latent constraints and has regret at most $\mathrm{Spread}(C^*)$. Thus, the theory focuses on reaching a covered region; $\mathrm{Spread}(C^*)$ captures the residual difficulty of choosing among feasible actions.

\textbf{Oracle access and the two settings.}
We distinguish the raw natural-language feedback $f_t\sim P_x$ from the oracle information that the learner can reliably extract from it. The raw feedback is the same in both regimes; the regimes differ only in the oracle interfaces available to the learner.

\begin{assumption}[Falsification access]
\label{ass:falsification-access}
The learner has access to a violation detector
$
\mathrm{Viol}:\mathcal{F}\times 2^{\mathcal{H}}\to\{0,1\}
$
such that $\mathrm{Viol}(f;C)=1$ only if $f$ contradicts some constraint already in $C$. In particular, feedback that reveals a missing constraint outside $C$ is discovery feedback, not a falsification of $C$.
\end{assumption}

\begin{assumption}[Identification access]
\label{ass:identification-access}
In addition to Assumption~\ref{ass:falsification-access}, the learner has an extractor
$
\mathrm{Extract}:\mathcal{F}\to\mathcal{H}\cup\{\emptyset\}
$
that returns a canonical violated constraint when one can be reliably identified, and returns $\emptyset$ otherwise.
\end{assumption}

For example, consider a tested child with $C_u=\{\texttt{year=1980s}\}$. Feedback ``I wanted something from the 90s'' can both refute $C_u$ and propose \texttt{year=1990s}. By contrast, feedback ``I wanted a comedy'' can reveal a missing genre constraint without contradicting $C_u$.

The \emph{falsification setting} assumes only Assumption~\ref{ass:falsification-access}, equivalently using a null extractor. The \emph{identification setting} assumes both Assumptions~\ref{ass:falsification-access} and~\ref{ass:identification-access}. The corresponding algorithmic variants are \textsc{TRACE-F} and \textsc{TRACE-I}. Quantitative reliability and soundness conditions, together with the oracle-level noise model used in our analysis, are stated in Appendix~\ref{app:theory-assumptions}.

\textbf{Algorithmic interfaces.}
Beyond these oracles, \textsc{TRACE} uses a constrained generator $Q_C$ for sampling actions satisfying a constraint set $C$, and a proposer $\mathrm{Propose}$ for selecting candidate constraints to test. These modules may be implemented by template decoders, LLM-based modules, database filters, constrained decoders, or domain-specific heuristics.

\section{Test-and-Refine Algorithm for Constraint Exploration}
\label{sec:algorithm}

We now describe the \emph{Test-and-Refine Algorithm for Constraint Exploration} (\textsc{TRACE}). The goal is to reach a constraint-tree node whose feasible region is contained in the user's latent feasible region. \textsc{TRACE} maintains a committed node $v$, proposes candidate constraints to refine $C_v$, and uses language feedback to decide whether each proposed refinement should be rejected or committed. The central design principle is that extracted identities serve as \emph{proposals} rather than facts: they bias which child is tested next, but enter the committed set only after passing a multi-round falsification test. %

\begin{algorithm}[t]
\caption{%
Test-and-Refine Algorithm for Constraint Exploration (\textsc{TRACE}).
}
\label{alg:robust}
\begin{algorithmic}[1]
\Require horizon $T$, test length $m_{\mathrm{exp}}$, max tests $N_{\max}$,
generators $\{Q_C\}$, detector $\mathrm{Viol}$, proposer $\mathrm{Propose}$, extractor $\mathrm{Extract}$
\State Initialize $v\leftarrow v_0$ with $C_{v_0}=\emptyset$;\,
$W\leftarrow\emptyset$;\, $R_{v'}\leftarrow\emptyset$ for every node $v'$;\,
$t\leftarrow 0$;\, $n\leftarrow 0$
\While{$t<T$ \textbf{and} $n<N_{\max}$}
    \State $c_{\mathrm{prop}}\leftarrow\mathrm{Propose}(v,W,R_v)$;\,
    \textbf{break if } $c_{\mathrm{prop}}=\emptyset$
    \State $u\leftarrow$ child of $v$ with
    $C_u=C_v\cup\{c_{\mathrm{prop}}\}$
    \State \textbf{if } $\mathcal{X}_u=\emptyset$ \textbf{ then } $R_v\leftarrow R_v\cup\{c_{\mathrm{prop}}\}$;\, \textbf{continue}
    \hfill \emph{empty region: exclude at $v$, no budget}
    \State $n\leftarrow n+1$;\, $\mathrm{pass}\leftarrow\mathsf{true}$
    \For{$s=1$ to $m_{\mathrm{exp}}$}
        \State \textbf{if } $t\ge T$ \textbf{ then } $\mathrm{pass}\leftarrow\mathsf{false}$;\, \textbf{break}
        \State $t\leftarrow t+1$;\, sample $x_t\sim Q_{C_u}$, receive $f_t$
        \State $c_{\mathrm{ext}}\leftarrow \mathrm{Extract}(f_t)$;\,
        \textbf{if } $c_{\mathrm{ext}}\neq\emptyset$ \textbf{ then }
        $W\leftarrow W\cup\{c_{\mathrm{ext}}\}$
        \hfill \emph{store extracted identity}
        \State \textbf{if } $\mathrm{Viol}(f_t;C_u)=1$ \textbf{ then }
        $\mathrm{pass}\leftarrow\mathsf{false}$;\, \textbf{break}
    \EndFor
    \State \textbf{if } $\mathrm{pass}=\mathsf{true}$ \textbf{ then } $v\leftarrow u$ \textbf{ else } $R_v\leftarrow R_v\cup\{c_{\mathrm{prop}}\}$
    \hfill \emph{commit, or exclude at $v$}
\EndWhile
\State \Return $\hat{x}\sim Q_{C_v}$
\end{algorithmic}
\end{algorithm}

\textbf{Feedback use and commitment.}
During a test of a proposed child $u$, feedback is used in two ways. The detector $\mathrm{Viol}(f_t;C_u)$ decides whether $u$ is falsified. When identification access is available, $\mathrm{Extract}(f_t)$ may add a candidate constraint to a global working set $W$ for future proposals. Crucially, entries in $W$ are never committed directly: they only affect which child is tested next.

\textbf{Test-before-commit.}
Given a proposed constraint $c_{\mathrm{prop}}$, \textsc{TRACE} forms the child $u$ with $C_u=C_v\cup\{c_{\mathrm{prop}}\}$ and samples from $Q_{C_u}$ for up to $m_{\mathrm{exp}}$ rounds. If any feedback contradicts $C_u$, the child is rejected and the algorithm remains at $v$. A per-node set $R_v$ records the candidates ruled out at $v$ (empty region or rejected child), so they are not proposed again at $v$; this ensures termination and never excludes a missing true constraint while $v$ is viable (Remark~\ref{app:rem:lazy-filter}). If the child survives all $m_{\mathrm{exp}}$ tests, \textsc{TRACE} commits to $u$ and moves deeper in the tree. At the end of the budget, it returns $\hat{x}\sim Q_{C_v}$, where $v$ is the deepest committed node at termination. If $\mathrm{Region}(C_v)\subseteq\mathrm{Region}(C^*)$, then $\hat{x}$ satisfies all latent requirements; the remaining simple regret is controlled by $\mathrm{Spread}(C^*)$. 

\textbf{Variants.}
Both variants propose from $\mathcal{H}_v:=\mathcal{H}\setminus(C_v\cup R_v)$, and $\mathrm{Propose}$ returns $\emptyset$ when $\mathcal{H}_v=\emptyset$. \textsc{TRACE-F} uses a null extractor and proposes uniformly from $\mathcal{H}_v$. \textsc{TRACE-I} first proposes from $W\cap\mathcal{H}_v$ when nonempty, and otherwise falls back to uniform proposal over $\mathcal{H}_v$. Both variants use the same falsification test before commitment; operational properties of $W$, the detector contract, and the role of the counters $t,s$ in Algorithm~\ref{alg:robust} are recorded in Appendix~\ref{app:algorithmic-remarks}.

\textbf{Implementation and theory bridge.}
\textsc{TRACE} is parser- and generator-agnostic: in finite recommendation or entity-deduction tasks, $Q_C$ can be a filtered catalog sampler, while in constrained generation it can be an LLM prompted with accumulated requirements or a constrained decoder. The detector and extractor can likewise be template-based, rule-based, LLM-based, or hybrid. Our experiments use lightweight decoders where possible, while Section~\ref{sec:theory} abstracts away these implementation choices and analyzes how many child tests are needed to reach a covered region.

\vspace{-0.1in}
\section{Theoretical Analysis}
\label{sec:theory}
\vspace{-0.1in}
We analyze the two oracle-access regimes introduced in Section~\ref{sec:tree}. The analysis abstracts away implementation details of the generator, detector, extractor, and proposer, and asks how many interaction rounds suffice for \textsc{TRACE} to reach a covered region. \textsc{TRACE-F} (Assumption~\ref{ass:falsification-access} only) uses the violation oracle for falsification and searches over the candidate class by blind proposal. \textsc{TRACE-I} (Assumptions~\ref{ass:falsification-access} and~\ref{ass:identification-access}) additionally uses extracted constraint identities to bias proposals. Both analyses share the same multi-sample falsification argument; they differ in how many child tests are needed before the latent feasible region is covered. Full proofs are deferred to Appendix~\ref{app:complete-theory}.

\vspace{-0.1in}
\subsection{Analysis model and guarantee}
\label{sec:analysis-setup}
\vspace{-0.05in}

We work with the finite canonical constraint class $\mathcal{H}\supseteq C^*$ from Section~\ref{sec:setup}. Throughout, $H:=|\mathcal{H}|$ is the candidate class size and $K:=|C^*|$ is the number of latent constraints; the bounds depend on these together with the extraction probability $p_{\mathrm{ext}}$ in the identification setting. The constrained generator $Q_C$ is supported on $\mathcal{X}(C):=\{x\in\mathcal{X}:\mathrm{Sat}(x,c)=1\ \forall c\in C\}$ and returns independent samples on each call. The identification-setting theorem (Assumption~\ref{ass:identification-access}) additionally requires %
sound extracted identities, i.e., $\mathrm{Extract}$ returns either a latent constraint or $\emptyset$.
Thus the theorem captures missed identity discoveries; %
spurious extracted identities are treated empirically in Section~\ref{sec:ablations-summary}. %
The implementation may run $\mathrm{Extract}$ on every feedback message; the theorem only requires that informative feedback events, i.e., violations observed while testing a spurious child, add missing true constraints to $W$ with probability at least $p_{\mathrm{ext}}$.

\textbf{Conditions for the guarantees.}
The formal assumptions are stated in Appendix~\ref{app:theory-assumptions}. Here we summarize their roles.

$\triangleright$ \textit{Refutational feedback:} For a node $v$, let
$q_v := \Pr[\mathrm{Viol}(f;C_v)=1 \mid x\sim Q_{C_v}, f\sim P_x]$
be the probability that one test sample falsifies the region. Viable regions are not incorrectly rejected by the contradiction detector: if $\Delta(v)=0$, then $q_v=0$. Conversely, spurious regions are falsifiable: if $\Delta(v)>0$, then $q_v\ge \min\{1,c_{\mathrm{fal}}\Delta(v)^\beta\}$ for constants $c_{\mathrm{fal}},\beta>0$. The exponent $\beta$ captures the resolution of the violation oracle (Assumption~\ref{ass:falsification-access}); larger $\beta$ means small gaps are harder to detect.

$\triangleright$ \textit{$\varepsilon$-separation (gap structure):} At the target tolerance $\varepsilon$, every child of a viable node either preserves a global optimum, $\Delta(u)=0$, or excludes all optima by margin $\Delta(u)>\varepsilon$. This is a structural condition on $(\mathcal{X}, U, \mathcal{H}, C^*)$, not on the generator, detector, or extractor; it ensures that a finite test window can separate viable children from spurious ones.

$\triangleright$ \textit{Identification-access progress:} Theorem~\ref{thm:identity} additionally assumes two conditions. The first is structural (Assumption~\ref{app:ass:identity-progress}): every constraint in $\mathcal{H}\setminus C^*$ with a nonempty region excludes all global optima with gap larger than $\varepsilon$ even when imposed alone, so adding it to any viable node yields an empty region or such a gap. In particular, the candidate class contains no \emph{zero-gap distractors}, i.e., constraints outside $C^*$ that preserve a global optimum when added to a viable node. The second concerns extraction (Assumption~\ref{app:ass:feedback}): on informative feedback, the extractor of Assumption~\ref{ass:identification-access} reveals a missing latent constraint with probability at least $p_{\mathrm{ext}}$. Together, the two conditions force every incorrect candidate to be incompatible with the true requirements (Appendix~\ref{app:theory-assumptions}). Reliable template decoders can support the extraction condition, but not the structural one, which depends only on $(\mathcal{X},U,\mathcal{H},C^*)$. When either condition fails, the $\widetilde{\mathcal{O}}(K/p_{\mathrm{ext}})$ guarantee of Theorem~\ref{thm:identity} does not apply.

\textbf{Coverage guarantee.}
We say that an algorithm enjoys an $(\varepsilon,\delta)$-coverage guarantee with budget $T$ if, with probability at least $1-\delta$, the deepest committed node $v$ at termination satisfies
\begin{align}
\mathrm{Region}(C_v)\subseteq \mathrm{Region}(C^*)
\end{align}
within $T$ interaction rounds. This target is representation-invariant: exact syntactic recovery of $C^*$ is sufficient, but not necessary. If $\hat{x}\sim Q_{C_v}$, coverage implies $\hat{x}\in\mathrm{Region}(C^*)$ and therefore
\begin{align}
\mathbb{E}[r_T]\le \mathrm{Spread}(C^*)+\delta.
\end{align}
The two theorems below establish coverage under the falsification and identification settings.

\vspace{-0.05in}
\subsection{Falsification- and identification-setting guarantees}
\vspace{-0.05in}

\textbf{Multi-sample falsification.}
The core observation is that a spurious child is unlikely to survive many independent feedback tests. We state the lemma for an arbitrary tested node; in Algorithm~\ref{alg:robust}, this node is typically the proposed child.

\begin{lemma}[Multi-sample falsification]
\label{lem:multi-sample}
Let $v$ be any node with $\Delta(v)>\varepsilon$. If the algorithm samples $x_s\sim Q_{C_v}$ for $m$ independent rounds and observes feedback $f_s\sim P_{x_s}$, then
\begin{align}
\Pr\left[\mathrm{Viol}(f_s;C_v)=0\ \text{for all }s=1,\ldots,m\right]\le\exp\left(-m\,c_{\mathrm{fal}}\varepsilon^\beta\right).
\end{align}
Consequently, if $m\ge \log(1/\delta')/(c_{\mathrm{fal}}\varepsilon^\beta)$, then a node with gap larger than $\varepsilon$ survives $m$ tests with probability at most $\delta'$.
\end{lemma}

\textbf{Falsification-setting guarantee.}
We first consider \textsc{TRACE-F}, which uses only the violation oracle of Assumption~\ref{ass:falsification-access} and proposes uniformly from the finite candidate class.

\begin{theorem}[Coverage in the falsification setting]
\label{thm:main-noisy}
Suppose the conditions of Section~\ref{sec:analysis-setup} hold for refutational feedback and $\varepsilon$-separation. Fix $\delta\in(0,1)$ and $\varepsilon\in(0,1]$. With
\begin{align}
N_{\max} = \left\lceil H\log\frac{2\max(K,1)}{\delta} \right\rceil,
\quad
m_{\mathrm{exp}} = \left\lceil
\frac{\log(2N_{\max}/\delta)}{c_{\mathrm{fal}}\varepsilon^\beta}
\right\rceil,
\quad
T = N_{\max}m_{\mathrm{exp}},
\end{align}
\textsc{TRACE-Falsification} enjoys an $(\varepsilon,\delta)$-coverage guarantee with $T=\widetilde{\mathcal{O}}(H/(c_{\mathrm{fal}}\varepsilon^\beta))$.
\end{theorem}

The bound has two main factors. The \emph{proposal cost} is $\widetilde{\mathcal{O}}(H)$: in the canonical finite-class proof, a sufficient route to coverage is to propose all latent constraints, which uniform proposal achieves in $\mathcal{O}(H\log K)$ calls. This cost is over proposal calls from $\mathcal{H}$, not over all ordered root-to-leaf paths in the constraint tree. The \emph{testing cost} is $\widetilde{\mathcal{O}}(1/(c_{\mathrm{fal}}\varepsilon^\beta))$: by Lemma~\ref{lem:multi-sample}, this many samples reject any child with gap larger than $\varepsilon$ with the desired confidence. Viable children survive because $q_u=0$ when $\Delta(u)=0$, while spurious children with $\Delta(u)>\varepsilon$ are refuted with probability at least $c_{\mathrm{fal}}\varepsilon^\beta$ per test. Once $\mathrm{Region}(C_v)\subseteq\mathrm{Region}(C^*)$, the returned action lies in the latent feasible region, giving $\mathbb{E}[r_T]\le \mathrm{Spread}(C^*)+\delta$.

\textbf{Identification-setting guarantee.}
\label{sec:identity-aware}
Theorem~\ref{thm:main-noisy} shows the falsification setting alone suffices for coverage, but requires blind proposal over the ambient candidate class. In a favorable identification regime where extracted identities make reliable progress toward missing latent constraints, we show that this candidate-class dependence can be replaced by the effective discovery cost $K/p_{\mathrm{ext}}$.

\begin{theorem}[Coverage in the identification setting]
\label{thm:identity}
Suppose the conditions of Section~\ref{sec:analysis-setup} hold, including identification-access progress. Fix $\delta\in(0,1)$ and $\varepsilon\in(0,1]$. With
\[
N_{\max}^{\mathrm{id}}
= \left\lceil
\frac{64}{p_{\mathrm{ext}}}\left(K+\log\frac{4}{\delta}\right)
\right\rceil,
\quad
m_{\mathrm{exp}}^{\mathrm{id}}
= \left\lceil
\frac{\log(2N_{\max}^{\mathrm{id}}/\delta)}{c_{\mathrm{fal}}\varepsilon^\beta}
\right\rceil,
\quad
T = N_{\max}^{\mathrm{id}}m_{\mathrm{exp}}^{\mathrm{id}},
\]
\textsc{TRACE-Identification} enjoys an $(\varepsilon,\delta)$-coverage guarantee with $T=\widetilde{\mathcal{O}}(K/(p_{\mathrm{ext}}\,c_{\mathrm{fal}}\varepsilon^\beta))$.
\end{theorem}

Theorem~\ref{thm:identity} replaces the blind-search cost $H$ by the effective discovery cost $K/p_{\mathrm{ext}}$. The intuition is that, under identification access, diagnostic feedback can add a missing true constraint to $W$ with probability $p_{\mathrm{ext}}$; once in $W$, that constraint is proposed before blind candidates and survives the same test window. Thus identification access is most useful when $H$ is large relative to $K/p_{\mathrm{ext}}$, and its advantage diminishes when the candidate class is small or extraction is unreliable. Section~\ref{sec:h-scaling} validates this regime distinction. Section~\ref{sec:ablations-summary} tests the case where the soundness clause of Assumption~\ref{ass:identification-access} is not satisfied (the extractor returns spurious constraints).

\section{Experiments}
\label{sec:experiments}
We evaluate the practical benefits of \textsc{TRACE}'s constraint-based formulation and test-before-commit mechanism through four questions:
 \textbf{(Q1)} %
How does \textsc{TRACE} compare with prompting methods under different feedback interfaces and common output caps?
\textbf{(Q2)} %
Does identification feedback reduce the number of candidate refinements tested as the candidate class grows?
\textbf{(Q3)} %
When does test-before-commit improve over directly accumulating extracted constraints? 
\textbf{(Q4)} How robust is \textsc{TRACE} to %
incorrect extracted identities?

\textbf{Setup.}
We evaluate on five environments spanning six task instances: LLF-Bench RecMovie~\citep{cheng2023llf}, a synthetic 20-Questions environment (\textsc{20Q}), Apple's Entity-Deduction Arena (\textsc{EDA Things})~\citep{zhang2024probing}, per-position \textsc{Mastermind}~\citep{li2024playing}, and LLF-Bench Poem~\citep{cheng2023llf} with Haiku and Tanka forms. The primary metric is \emph{env.success} over $100$ seeds per cell; for RecMovie this refers to the returned recommendation list (an interaction-budget any-success diagnostic is in Appendix~\ref{app:recmovie-full}). We compare \textsc{TRACE-F}, which uses binary falsification with uniform proposals, with \textsc{TRACE-I}, which additionally prioritizes constraints extracted from past feedback. Prompt-based baselines \textsc{Random}, \textsc{DirectLLM}, \textsc{Reflexion}~\citep{shinn2023reflexion}, and \textsc{Self-Refine}~\citep{madaan2023selfrefine} use Qwen3-4B~\citep{yang2025qwen3} under fixed per-environment interaction budgets ($T=20$ on RecMovie) that are not matched to \textsc{TRACE}'s default configuration (see Section~\ref{sec:budget-curve} for equal budgets); structured same-interface baselines use the same decoded feedback as \textsc{TRACE} but remove the test-before-commit step. Throughout this section, brackets in result tables denote Wilson $95\%$ confidence intervals; hyperparameters, $\mathrm{Viol}/\mathrm{Extract}$ implementation, full-table intervals, attempts, and soft-utility diagnostics are deferred to Appendix~\ref{app:experimental-details}.

\begin{table}[t]
\centering
\small
\setlength{\tabcolsep}{3pt}
\caption{Main comparison under three oracle-access regimes. Random does not consume feedback. Wilson $95\%$ intervals in Tables~\ref{tab:app-main-ci} and~\ref{tab:app-same-oracle-ci}; \textsc{Mastermind}/Haiku/Tanka in Appendix~\ref{app:saturated-tasks}.}
\label{tab:main-comparison}
\begin{tabular}{l rrr rrr rrr}
\toprule
& \multicolumn{3}{c}{\textbf{RecMovie}}
& \multicolumn{3}{c}{\textbf{\textsc{20Q}}}
& \multicolumn{3}{c}{\textbf{\textsc{EDA}}} \\
\cmidrule(lr){2-4}\cmidrule(lr){5-7}\cmidrule(lr){8-10}
\textbf{Method} & Natural & Falsif. & Identif. & Natural & Falsif. & Identif. & Natural & Falsif. & Identif. \\
\midrule
Random       & $\phantom{1}5\%$ & --   & --    & $41\%$ & --   & --     & $74\%$ & --   & --     \\
DirectLLM    & $17\%$ & $27\%$ & $40\%$ & $12\%$ & $43\%$ & $\mathbf{100\%}$ & $43\%$ & $84\%$ & $\mathbf{100\%}$ \\
Reflexion    & $26\%$ & $35\%$ & $42\%$ & $11\%$ & $29\%$ & $\mathbf{100\%}$ & $55\%$ & $69\%$ & $\mathbf{100\%}$ \\
Self-Refine  & $21\%$ & $26\%$ & $32\%$ & $73\%$ & $19\%$ & $\mathbf{100\%}$ & $78\%$ & $40\%$ & $\mathbf{100\%}$ \\
\textbf{\textsc{TRACE}}  & -- & $\mathbf{51\%}$ & $\mathbf{90\%}$ & -- & $\mathbf{84\%}$ & $\mathbf{100\%}$ & -- & $\mathbf{87\%}$ & $99\%$ \\
\bottomrule
\end{tabular}
\end{table}

\subsection{Main comparison: explicit constraint structure helps}
\label{sec:main-results}

Table~\ref{tab:main-comparison} reports env.success on three task instances under three oracle-access regimes: \emph{Natural} feedback, where prompt-based methods consume the environment's raw text; \emph{Falsification} access, corresponding to Assumption~\ref{ass:falsification-access}; and \emph{Identification} access, corresponding to Assumptions~\ref{ass:falsification-access}+\ref{ass:identification-access}. Within each regime, all methods share the same oracle interface, so cross-method comparisons within a regime are matched. \textsc{TRACE-F} and \textsc{TRACE-I} appear under the Falsification and Identification columns, respectively; they are the test-before-commit instantiations of the falsification and identification settings, not competing variants.

At each method's default budget, within each oracle regime, \textsc{TRACE} improves over the strongest same-regime prompt baseline: \textsc{TRACE-F} by $16$pp on RecMovie ($51\%$ vs.\ $35\%$), $41$pp on \textsc{20Q}, and $3$pp on \textsc{EDA}; \textsc{TRACE-I} by $48$pp on RecMovie ($90\%$ vs.\ $42\%$), and by at least $31$pp when all methods share an equal output cap (Section~\ref{sec:budget-curve}). 
On \textsc{20Q} and \textsc{EDA}, identity-feedback methods reach 99--100\%, leaving little headroom for additional gains in these configurations. Self-Refine's success on \textsc{20Q} drops from 73\% under natural feedback to 19\% under binary projection, illustrating the importance of the information retained in the feedback interface.
Three additional task instances (\textsc{Mastermind} with $K=3$, Haiku, and Tanka) are deferred to Appendix~\ref{app:saturated-tasks}, where the qualitative pattern is identical: \textsc{TRACE-I} reaches $99$--$100\%$ while natural-feedback baselines remain below $30\%$. The candidate-class scaling on \textsc{Mastermind} is studied in Section~\ref{sec:h-scaling}.

\subsection{Equal output budgets and interaction cost}
\label{sec:budget-curve}

Table~\ref{tab:main-comparison} uses method-specific default budgets. We therefore compare \textsc{TRACE-I} with prompting and sampling controls on RecMovie, which provides a non-saturated comparison under identification feedback. 
Figure~\ref{fig:equal-budget-recmovie}a evaluates common caps of $B\in\{20,60\}$ environment-evaluated outputs, scoring one returned output per seed: \textsc{TRACE-I} ($m_{\mathrm{exp}}=5$) uses at most $B-1$ outputs for tests and returns one within the cap, while each iterative baseline returns its first accepted output or its last output at the cap. All baselines receive the same identity projection as \textsc{TRACE-I}; to run them with $60$ outputs, DirectLLM-W8 and Self-Refine-W8 keep only the eight most recent interactions, and Reflexion-NoCycle disables early trial termination on repeated recommendations. \textsc{TRACE-I} reaches $73\%$ at $B=20$ and $86\%$ at $B=60$, versus at most $42\%$ and $48\%$ for the baselines evaluated at each cap; Best-of-$N$ with an oracle selector reaches $17\%$ and $21\%$.

Each test round shows a complete recommendation rather than asking the user to reconfirm a constraint, and smaller test windows trade reliability for fewer rounds (Figure~\ref{fig:equal-budget-recmovie}b): with $m_{\mathrm{exp}}=1$, \textsc{TRACE-I} reaches $65\%$ with $4.43$ feedback rounds on average. Under a tight cap a larger window uses up more of the budget ($78\%$ for $m_{\mathrm{exp}}=3$ vs.\ $73\%$ for $m_{\mathrm{exp}}=5$ at $B=20$), whereas at $B=60$ it reaches $86\%$. Appendix~\ref{app:equal-budget} gives the full protocol, intervals, and costs, and shows that seeding the root with up to two oracle-provided constraints raises success to $90\%$ at $B=20$.

\begin{figure}[h]
\centering
\includegraphics[width=\linewidth]{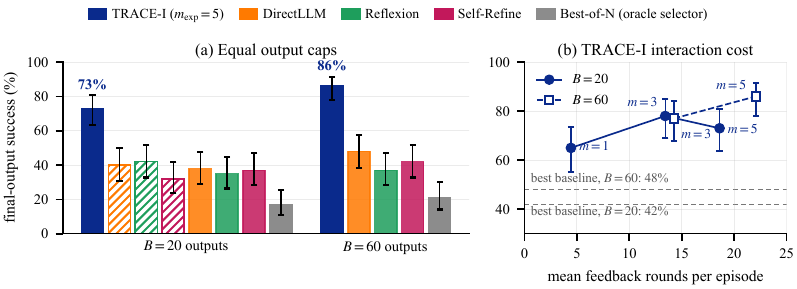}
\caption{RecMovie under equal output budgets ($100$ seeds; final-output success with Wilson $95\%$ intervals). (a) Common caps of $B=20$ and $B=60$ environment-evaluated outputs. Hatched bars are the Table~\ref{tab:main-comparison} baseline configurations, which were run only with $T=20$; solid baseline bars are the W8/NoCycle variants run with $60$ outputs. (b) \textsc{TRACE-I} success versus mean feedback rounds for test windows $m=m_{\mathrm{exp}}\in\{1,3,5\}$; dashed lines mark the strongest baseline at each cap.}
\label{fig:equal-budget-recmovie}
\end{figure}

\subsection{Identity feedback reduces blind search}
\label{sec:h-scaling}

Theorem~\ref{thm:identity} predicts that identity-aware proposal can replace blind dependence on the ambient candidate class size $H$ by the effective discovery cost $K/p_{\mathrm{ext}}$. We test this in a controlled \textsc{Mastermind} environment with candidate class $\mathcal{H}=\{(i,c): i\in[K],\, c\in\mathrm{palette}\}$ and $|\mathcal{H}|=K\cdot n_{\mathrm{colors}}$. Table~\ref{tab:h-scaling} fixes $K=3$ (with $N_{\max}=80$ enlarged from the headline $N_{\max}=15$ to give \textsc{TRACE-F} budget headroom at large $|\mathcal{H}|$) and varies $n_{\mathrm{colors}}\in\{3,6,10,15,20\}$, thereby increasing $H$ while keeping the number of latent constraints fixed; \emph{attempts} is the number of proposed children actually tested. Under the uniform binary proposer, success drops from $100\%$ at $|\mathcal{H}|=9$ to $74\%$ at $|\mathcal{H}|=60$, and mean attempts grow from $9.08$ to $50.82$. In contrast, the decoded identity proposer remains at $99$--$100\%$ success across the same sweep and uses substantially fewer attempts ($3.91$ to $18.08$). %
These results support the proposal-efficiency advantage of identification access as the candidate class grows. 

\begin{table}[h]
\centering
\small
\setlength{\tabcolsep}{3pt}
\caption{$H$-sweep on \textsc{Mastermind} ($K=3$, $m_{\mathrm{exp}}=15$, $N_{\max}=80$). Wilson $95\%$ intervals in Table~\ref{tab:app-h-ci}.}
\label{tab:h-scaling}
\begin{tabular}{rrrrrr}
\toprule
\textbf{$n_{\mathrm{colors}}$} & \textbf{$|\mathcal{H}|$} & \multicolumn{2}{c}{\textbf{\textsc{Falsification}}} & \multicolumn{2}{c}{\textbf{\textsc{Identification}}} \\
\cmidrule(lr){3-4} \cmidrule(lr){5-6}
& & \textbf{env.succ.} & \textbf{attempts} & \textbf{env.succ.} & \textbf{attempts} \\
\midrule
$3$  &  $9$ & $\mathbf{100\%}$ &  $\phantom{1}9.08$ & $\mathbf{100\%}$ & $\phantom{1}3.91$ \\
$6$  & $18$ & $\mathbf{100\%}$ & $18.75$ & $\mathbf{100\%}$ & $\phantom{1}4.84$ \\
$10$ & $30$ & $\phantom{1}98\%$ & $29.35$ & $\mathbf{100\%}$ & $\phantom{1}6.77$ \\
$15$ & $45$ & $\phantom{1}88\%$ & $45.38$ & $\mathbf{100\%}$ & $10.54$ \\
$20$ & $60$ & $\phantom{1}74\%$ & $50.82$ & $\phantom{1}99\%$ & $18.08$ \\
\bottomrule
\end{tabular}
\end{table}

\subsection{Why test-before-commit matters}
\label{sec:ablations-summary}

\begin{table}[b]
\centering
\small
\begin{minipage}[t]{0.44\linewidth}
\centering
\setlength{\tabcolsep}{4pt}
\caption{Direct-Commit baselines: best of $\{\text{AND}, \text{OR}, \text{RANKED}\}$ per environment; variants in Appendix~\ref{app:recmovie-full}.}
\label{tab:structure-necessity}
\begin{tabular}{lrrr}
\toprule
\textbf{Method} & \textbf{Rec.} & \textbf{\textsc{20Q}} & \textbf{\textsc{EDA}} \\
\midrule
\textbf{\textsc{TRACE-I}}      & $\mathbf{90\%}$ & $\mathbf{100\%}$ & $99\%$ \\
Direct-Commit                  & $23\%$ & $\mathbf{100\%}$ & $\mathbf{100\%}$ \\
\bottomrule
\end{tabular}
\end{minipage}\hfill
\begin{minipage}[t]{0.52\linewidth}
\centering
\setlength{\tabcolsep}{3pt}
\caption{False-identity robustness on RecMovie ($p_{\mathrm{false}}=0$ row matches Table~\ref{tab:main-comparison}).}
\label{tab:false-identity-main}
\begin{tabular}{lrr}
\toprule
\textbf{$p_{\mathrm{false}}$} & \textbf{\textsc{TRACE-I}} & \textbf{Direct-Commit} \\
\midrule
$0.0$ & $\mathbf{90\%}$ {\scriptsize [$83$,$94$]} & $23\%$ {\scriptsize [$16$,$32$]} \\
$0.2$ & $\mathbf{74\%}$ {\scriptsize [$65$,$82$]} & $16\%$ {\scriptsize [$10$,$24$]} \\
$0.5$ & $\mathbf{60\%}$ {\scriptsize [$50$,$69$]} & $\phantom{1}8\%$ {\scriptsize [$4$,$15$]} \\
\bottomrule
\end{tabular}
\end{minipage}
\end{table}

We examine when test-before-commit improves performance by comparing \textsc{TRACE-I} with direct accumulation using the same decoder, followed by controlled corruption of extracted identities.

\textbf{Same-decoder direct accumulation.} %
Table~\ref{tab:structure-necessity} compares \textsc{TRACE-I} with Direct-Commit baselines that immediately accumulate constraints from the same decoded feedback. Under the reported default configurations, \textsc{TRACE-I} reaches $90\%$ success on RecMovie, compared with $23\%$ for the strongest direct-commit variant. On \textsc{20Q} and \textsc{EDA Things}, where identity hints are essentially clean, both approaches reach $99$--$100\%$. These saturated cases leave little headroom for further gains from test-before-commit. We next examine the effect of incorrect identities through controlled corruption.

\textbf{False extracted identities.} In Table~\ref{tab:false-identity-main}, we inject false identities into RecMovie by replacing each decoded feedback predicate with a wrong same-dimension predicate with probability $p_{\mathrm{false}}$. \textsc{TRACE-I} uses corrupted identities only as proposals, whereas direct commit immediately accumulates them. As corruption increases, both methods degrade, but \textsc{TRACE-I} remains more robust: at $p_{\mathrm{false}}=0.5$, it still reaches $60\%$ env.success, while the best direct-commit variant reaches only $8\%$.  These results support testing extracted constraints before commitment when identity errors are possible.

\textbf{Additional diagnostics.} %
Component ablations on RecMovie distinguish proposal quality from detector reliability. With the rule-based detector fixed, an LLM-based atomic proposer increases interaction-budget any-success from $57\%$ under uniform proposals to $74\%$. Replacing the rule-based detector with an LLM detector instead reduces \textsc{TRACE-F}'s any-success from $57\%$ to $21\%$. Appendix~\ref{app:experimental-details} reports these ablations, soft-utility results, direct-commit variants, feedback-omission robustness, and a complementary \textsc{Mastermind} $K$-sweep; Appendix~\ref{app:trace-recmovie} provides an example execution trace.

\subsection{\textsc{TravelPlanner}: cross-LLM identification}
\label{sec:travelplanner}

\begin{table}[t]
\centering
\small
\setlength{\tabcolsep}{3pt}
\caption{\textsc{TRACE-I} on \textsc{TravelPlanner}, $n{=}50$. \emph{Identification}: committed $C_v\supseteq C^*$. \emph{Final Pass}: plan passes every env check.}
\label{tab:travelplanner}
\begin{tabular}{lrr}
\toprule
\textbf{LLM} & \textbf{Identification} & \textbf{Final Pass} \\
\midrule
Qwen3-4B    & $84\%$ {\scriptsize [$72$, $92$]} & $16\%$ {\scriptsize [$\phantom{0}8$, $29$]} \\
DeepSeek-V4-Flash & $\mathbf{86\%}$ {\scriptsize [$74$, $93$]} & $\mathbf{90\%}$ {\scriptsize [$79$, $96$]} \\
\bottomrule
\end{tabular}
\end{table}
We use a feedback-enabled adaptation of \textsc{TravelPlanner}'s sole-planning setting~\citep{xie2024travelplanner}, supplying reference information and evaluating complete candidate plans. User-stated requirements form the candidate constraint class $\mathcal{H}$, and per-check evaluation messages provide feedback for testing candidate refinements. We evaluate on the same $50$ validation queries with two LLM generators: Qwen3-4B~\citep{yang2025qwen3} served locally, and DeepSeek-V4-Flash~\citep{deepseekai2026deepseekv4}. We report the rate of committing all represented constraints, $C_v\supseteq C^*$, and final-plan success under this adapted protocol.
The two generator configurations yield similar rates of committing all represented constraints, $84\%$ and $86\%$, while their final-plan pass rates differ substantially, at $16\%$ and $90\%$. These results distinguish explicit constraint commitment from producing a plan that passes every evaluation check. Both metrics are needed to assess this adaptation.

\bibliography{mybib}
\bibliographystyle{iclr2027_conference}

\clearpage
\appendix

\section{Related Work}
\label{sec:related}

\paragraph{Learning from language feedback.}
The emerging language-feedback learning paradigm is tightly coupled with the rapid progress of language models. A first line of work treats natural-language feedback as a verbal reinforcement signal at inference time: \textsc{Reflexion}~\citep{shinn2023reflexion} maintains a self-reflective memory across trials, \textsc{Self-Refine}~\citep{madaan2023selfrefine} iteratively edits an output using its own critiques, and \textsc{TextGrad}~\citep{yuksekgonul2025optimizing} performs gradient-like updates over prompt tokens guided by language feedback. Tree-based agents such as \textsc{LATS}~\citep{zhou2024language} additionally search over reasoning steps, but rely on numeric value estimates rather than constraint signals. A complementary line studies whether large language models can themselves explore in interactive settings~\citep{krishnamurthy2024can}, and a separate body of work develops constraint-rich language tasks that motivate our setup, including \textsc{LLF-Bench}~\citep{cheng2023llf}, \textsc{COLLIE}~\citep{yao2024collie}, \textsc{ChinaTravel}~\citep{shao2026chinatravel}, and \textsc{TravelPlanner}~\citep{xie2024travelplanner}. Two recent works establish formal guarantees for language-feedback learning under different formalisms: \citet{misra2024provable} give $\sqrt{T}$-regret bounds under hindsight instruction feedback, and \textsc{HELiX}~\citep{xu2025helix} studies consistency-based elimination over a global hypothesis class. Our focus is orthogonal: we model a constraint-bearing regime where feedback provides local falsification and optional atomic identity evidence, and the algorithmic state is a partially committed constraint set rather than a full global hypothesis.

\paragraph{Pure exploration and identification problems.}
Our analysis belongs to the pure-exploration tradition in stochastic decision making, where the goal is to cover or identify a target object with high probability rather than to minimize cumulative reward. Best-arm identification under the fixed-budget setting was introduced by \citet{audibert2010best}, with subsequent work establishing matching upper and lower bounds via sequential halving and multi-sample concentration arguments~\citep{karnin2013almost,chen2017towards}. Among extensions, PAC top-$K$ best-arm identification~\citep{kalyanakrishnan2012pac,bubeck2013multiple} and feasible-arms identification~\citep{katz2018feasible,katz2019top} share with our setting the goal of recovering a set-like target from a finite candidate class. We extend this perspective to language-feedback learning along two axes: progress comes from gap-dependent \emph{falsification} of candidate constraint sets rather than reward concentration on individual arms, and identity-bearing feedback can propose atomic constraints that must still be tested before commitment. More broadly related are bandit problems over structured action spaces with numeric rewards~\citep{slivkins2019introduction,bubeck2011x,kleinberg2019bandits}.

\section{Complete Theoretical Analysis}
\label{app:complete-theory}

This appendix provides the complete theoretical analysis for constraint tree exploration. Section~\ref{app:theory-assumptions} restates the assumptions used by the proofs. Section~\ref{app:basic-lemmas} gives an exact-gap survival lemma and a safe-commitment lemma shared by both coverage proofs. Section~\ref{app:strict-regime} analyzes an idealized sound-proposer regime. Section~\ref{app:noisy-regime} restates Theorem~\ref{thm:main-noisy} in full as Theorem~\ref{app:thm:noisy} and proves it for \textsc{TRACE-F}. Section~\ref{app:identity-regime} restates Theorem~\ref{thm:identity} in full as Theorem~\ref{app:thm:identity} and proves it for \textsc{TRACE-I}. Section~\ref{app:ucb-bonus} records an anytime gap-UCB bound derived from the same survival argument.

\subsection{Assumptions}
\label{app:theory-assumptions}

We use the notation from Section~\ref{sec:formulation}. In particular, each tree node $v$ is associated with a constraint set $C_v$, feasible region $\mathcal{X}_v=\mathrm{Region}(C_v)$, node value $U(v):=\max_{x\in\mathcal{X}_v}U(x)$, and node gap $\Delta(v):=U^*-U(v)$. We write $\mathrm{Viol}(f;C_v)=1$ when the detector concludes that feedback $f$ contradicts at least one constraint in the tested set $C_v$; feedback about a missing but uncommitted constraint does not fire this predicate. The assumptions below also require a refutational feedback condition: if a tested committed set excludes all optima, feedback must contradict that set with positive probability. We write $\mathcal{H}$ for the finite candidate class used by the proposer. Throughout, we assume $K\ge 1$; if $K=0$, then $\mathrm{Region}(C^*)=\mathcal{X}$ and the root already satisfies the coverage target.

\paragraph{Oracle-level noise model.}
The raw message $f\sim P_x$ may be stochastic, but the analysis imposes reliability conditions on the oracle outputs: the firing probability of $\mathrm{Viol}$ and the extraction probability of $\mathrm{Extract}$. The theoretical model allows false negatives but rules out false positives. This one-sided noise model supports the multi-sample falsification argument, while Section~\ref{sec:ablations-summary} empirically evaluates robustness to spurious extracted identities.

\begin{assumption}[Refutational feedback and extraction]
\label{app:ass:feedback}
For any node $v$ for which $Q_{C_v}$ is defined, let
\[
q_v
:=
\Pr\!\left[
\mathrm{Viol}(f;C_v)=1
\,\middle|\,
x\sim Q_{C_v},\ f\sim P_x
\right].
\]
There exist constants $c_{\mathrm{fal}}>0$ and $\beta>0$ such that
\[
q_v=0
\quad\text{if}\quad
\Delta(v)=0,
\qquad\qquad
q_v\ge \phi(\Delta(v)):=\min\!\left\{1,\,c_{\mathrm{fal}}\Delta(v)^\beta\right\}
\quad\text{if}\quad
\Delta(v)>0.
\]
The equality $q_v=0$ is the no-false-alarm side of the detector contract. The lower bound for $\Delta(v)>0$ is an additional refutation-power assumption on the feedback distribution: a bad tested region must produce feedback that contradicts some committed constraint with probability controlled by its gap. Without this assumption, a user who only mentions missing uncommitted requirements could leave a bad committed constraint unfalsified. Without loss of generality, $c_{\mathrm{fal}}\le 1$: replacing $c_{\mathrm{fal}}$ by $\min\{c_{\mathrm{fal}},1\}$ preserves this condition, and the normalization lets the rounding in $m_{\mathrm{exp}}$ be absorbed into the $\mathcal{O}(\cdot)$ budget bounds below. We call this condition on $q_v$ the \emph{refutational-feedback condition}. Every result in this appendix except Theorem~\ref{app:thm:identity} uses only this condition; the extraction clauses (a)--(b) below are needed only in the identification setting, and \textsc{TRACE-F} uses a null extractor. In the identification setting, the extractor $\mathrm{Extract}:\mathcal{F}\to \mathcal{H}\cup\{\emptyset\}$ satisfies the following two conditions.
\begin{enumerate}[label=(\alph*)]
\item \textbf{Extraction reliability.} Conditional on the event $\mathrm{Viol}(f_t;C_v)=1$ for the constraint set currently being tested, the extractor returns a violated canonical constraint with probability at least $p_{\mathrm{ext}}>0$, independently across rounds. It returns $\emptyset$ otherwise.

\item \textbf{Soundness.} If $\mathrm{Extract}(f_t)=c\neq\emptyset$, then $c\in C^*$ and $\mathrm{Sat}(x_t,c)=0$. In particular, extraction never returns a spurious constraint.
\end{enumerate}
\end{assumption}

\begin{assumption}[Constrained generation]
\label{app:ass:generation}
For every constraint set $C$ considered by the algorithm with nonempty feasible region, there exists a distribution $Q_C$ supported on $\mathrm{Region}(C)=\{x\in\mathcal{X}:\mathrm{Sat}(x,c)=1\ \forall c\in C\}$. Each call to $Q_C$ returns an independent sample.
\end{assumption}

\begin{assumption}[Sound proposer]
\label{app:ass:sound-proposer}
The proposer $\mathrm{Propose}$ is a possibly randomized map from nodes to $\mathcal{H}\cup\{\emptyset\}$ satisfying:
\begin{enumerate}[label=(\alph*)]
\item \textbf{Completeness.} If $C_v\subsetneq C^*$, then $\mathrm{Propose}(v)\in C^*\setminus C_v$.

\item \textbf{Termination marker.} If $C_v=C^*$, then $\mathrm{Propose}(v)=\emptyset$.
\end{enumerate}
\end{assumption}

Assumption~\ref{app:ass:sound-proposer} is an idealization. It is useful as a warm-up because it isolates the role of tree depth and within-node spread. The main robust analysis below replaces it with a noisy proposer.

\begin{assumption}[Candidate class and noisy proposer]
\label{app:ass:noisy-proposer}
The finite canonical constraint class $\mathcal{H}$ from Section~\ref{sec:setup} satisfies $C^*\subseteq\mathcal{H}$ and $|\mathcal{H}|=H\ge K$. For every $C\subseteq\mathcal{H}$ with nonempty feasible region, the generator $Q_C$ is defined as in Assumption~\ref{app:ass:generation}. At each proposal call from node $v$, the proposer returns an element of $\mathcal{H}_v=\mathcal{H}\setminus(C_v\cup R_v)$ sampled uniformly at random, independently of the past given $(v,R_v)$, where $R_v$ is the per-node lazy filter of Algorithm~\ref{alg:robust}. If $\mathcal{H}_v=\emptyset$, it returns $\emptyset$.
\end{assumption}

\begin{remark}[Per-node lazy filter]
\label{app:rem:lazy-filter}
The set $R_v$ in Algorithm~\ref{alg:robust} records the candidates that gave an empty region or were rejected at $v$; it starts empty at each node, and proposals are drawn from $\mathcal{H}_v=\mathcal{H}\setminus(C_v\cup R_v)$. At a viable node, a missing true constraint never enters $R_v$: its child contains a global optimum, so the region is nonempty and, by Assumption~\ref{app:ass:feedback}, the child is never falsified (a test cut short by budget exhaustion also adds its candidate to $R_v$, but ends the run). Hence missing true constraints remain available at viable nodes, which is what the proofs of Theorems~\ref{app:thm:noisy} and~\ref{app:thm:identity} use. Since each proposal at $v$ either enters $R_v$ or is committed, at most $|\mathcal{H}\setminus C_v|+1$ proposal calls occur at each node, so Algorithm~\ref{alg:robust} terminates.
\end{remark}

\begin{assumption}[$\varepsilon$-separated spurious children]
\label{app:ass:eps-sep}
Fix $\varepsilon\in(0,1]$. For every viable node $v$ (i.e., $\Delta(v)=0$) and every $c\in\mathcal{H}\setminus C_v$ whose child $u=v\cup\{c\}$ has a nonempty feasible region $\mathcal{X}_u$, this child satisfies either $\Delta(u)=0$ (it preserves a global optimum) or $\Delta(u)>\varepsilon$ (it excludes every optimum; we call such a child $\varepsilon$-spurious).
\end{assumption}

Assumption~\ref{app:ass:eps-sep} is a structural condition on $(\mathcal{X}, U, \mathcal{H}, C^*)$, not on the generator $Q_C$, detector $\mathrm{Viol}$, extractor $\mathrm{Extract}$, or proposer $\mathrm{Propose}$. It plays the role of a sub-optimality gap: without it, children with arbitrarily small but positive $\Delta(u)$ would require unboundedly many samples to falsify. The gap $\Delta(u)$ is defined only when $\mathcal{X}_u\neq\emptyset$. Children with empty regions are never tested by Algorithm~\ref{alg:robust}, so Assumptions~\ref{app:ass:eps-sep} and~\ref{app:ass:identity-progress} impose no condition on them.

\begin{assumption}[No zero-gap distractors]
\label{app:ass:identity-progress}
For the identification-setting analysis, every $c\in\mathcal{H}\setminus C^*$ with $\mathrm{Region}(c)\neq\emptyset$ satisfies $\Delta(\{c\})>\varepsilon$; that is, imposing $c$ alone already excludes every global optimum with margin $\varepsilon$.
\end{assumption}

Assumption~\ref{app:ass:identity-progress} is the structural condition used by Theorem~\ref{app:thm:identity}. It extends to every viable node $v$: a true constraint gives a child that keeps the optimum of $v$, while for $c\notin C^*$ with a nonempty child, $\mathrm{Region}(C_v\cup\{c\})\subseteq\mathrm{Region}(c)$ gives $\Delta(v\cup\{c\})\ge\Delta(\{c\})>\varepsilon$. Hence it implies Assumption~\ref{app:ass:eps-sep}, and, since $K\ge1$ makes the root such a node, it is equivalent to requiring the same property at every viable node that does not yet contain $C^*$. It depends only on $(\mathcal{X},U,\mathcal{H},C^*)$, and extractor reliability cannot ensure it; without it, an identification-aware proposer could spend budget committing harmless but non-informative constraints, reintroducing dependence on their number.

The extraction progress used by Theorem~\ref{app:thm:identity} follows from Assumption~\ref{app:ass:feedback}(a)--(b). Suppose a test of a child $u$ of a viable node $v$ starts with working set $W_{\mathrm{start}}\subseteq C_v$ and ends with a violation. Conditional on the history and on $u$, the rounds of the test are independent, so the feedback at the first violation has the one-round law conditional on a violation, and by (a) the extractor then returns a constraint with probability at least $p_{\mathrm{ext}}$. By (b) this constraint lies in $C^*$ and is violated by the sampled action, which satisfies $C_u$; hence it lies in $C^*\setminus C_u\subseteq C^*\setminus(C_v\cup W_{\mathrm{start}})$, so $|C_v\cup W|$ has increased during the test, at this round or at an earlier round that returned the same constraint. No independence between different tests is needed.

The same conditions restrict the instances covered by Theorem~\ref{app:thm:identity}: together with Assumption~\ref{app:ass:feedback}, Assumption~\ref{app:ass:identity-progress} forces $\mathrm{Region}(c)\cap\mathrm{Region}(C^*)=\emptyset$ for every $c\in\mathcal{H}\setminus C^*$. Otherwise the node $D=C^*\cup\{c\}$ is nonempty with $\Delta(D)\ge\Delta(\{c\})>\varepsilon$, so the refutational-feedback condition gives $q_D>0$; but every action in $\mathrm{Region}(D)$ satisfies all true constraints, so by (b) no extraction can succeed at $D$, contradicting (a). The identification guarantee therefore applies when every incorrect candidate is incompatible with the true requirements, as for wrong values of single-valued attributes; Theorem~\ref{app:thm:noisy} does not use (a)--(b) and is unaffected.

\subsection{Algorithmic remarks}
\label{app:algorithmic-remarks}

The remarks in this subsection record operational properties of Algorithm~\ref{alg:robust} that are referenced from the main text. They are not used in any proof.

\begin{remark}[Working-set persistence and bounded growth]
\label{app:rem:W-persistence}
The working set $W$ in Algorithm~\ref{alg:robust} is initialized once and updated only by union; it is never cleared. Four observations characterize its behavior. (a) $|W|\le H$ since $W\subseteq\mathcal{H}$. (b) The identification proposer reads $W\cap\mathcal{H}_v$ rather than $W$ directly, so already-committed and already-excluded constraints are filtered at proposal time. (c) Under sound extraction (Assumption~\ref{app:ass:feedback}(b)), every $c\in W$ is a true latent constraint; once proposed at a viable node $v$, the resulting child has $\Delta=0$, is never falsified by the no-false-alarm clause of Assumption~\ref{app:ass:feedback}, and is committed at the end of its $m_{\mathrm{exp}}$-round test window whenever the remaining budget covers that window. (d) Under empirically unsound extraction, false entries in $W$ may be repeatedly proposed. Proposing such an entry at a viable node yields an empty child, which is excluded without a test, a zero-gap child, which may be committed but still contains a global optimum, or, by Assumption~\ref{app:ass:eps-sep}, a child with $\Delta>\varepsilon$. Lemma~\ref{app:lem:safety} allows any history-dependent proposer, so with $m_{\mathrm{exp}}$ as in Theorems~\ref{app:thm:noisy} and~\ref{app:thm:identity} and a detector that still satisfies the refutational-feedback condition, no child with positive gap is committed with probability at least $1-\delta/2$; false entries then cost test budget while the committed node remains viable. This controls harmful commitments only; the $K/p_{\mathrm{ext}}$ guarantee of Theorem~\ref{app:thm:identity} additionally requires sound extraction and Assumption~\ref{app:ass:identity-progress}. This is the algorithmic backing for the false-identity robustness check in Section~\ref{sec:ablations-summary}. Combined with the per-node lazy filter (Remark~\ref{app:rem:lazy-filter}), \textsc{TRACE-I} cannot loop on a spurious extracted identity once it has been falsified at $v$: the lazy filter excludes it from subsequent proposals at the same node.
\end{remark}

\begin{remark}[Detector contract, operational reading]
\label{app:rem:detector-contract}
Per Assumption~\ref{app:ass:feedback}, $\mathrm{Viol}$ is used only for falsification: feedback that reveals a missing requirement may update $W$ via $\mathrm{Extract}$, but does not reject the tested child unless it contradicts a constraint already in $C_u$. Equivalently, the detector fires on contradictions with the tested set, never on discovery feedback alone. As an illustration, if $C_u=\{\mathrm{genre}=\mathrm{comedy}\}$, then ``I did not want a comedy'' contradicts $C_u$ and fires $\mathrm{Viol}$, whereas ``I also wanted something from the 1990s'' is discovery feedback and must not fire $\mathrm{Viol}$.
\end{remark}

\begin{remark}[Counters in Algorithm~\ref{alg:robust}]
\label{app:rem:counters}
The variable $t$ is the global interaction-round counter running against the budget $T$; the inner-loop index $s$ counts tests of the current child running against the per-child cap $m_{\mathrm{exp}}$. The two counters are independent: $s$ resets on each new child while $t$ is monotone over the entire run. Both caps are needed because the algorithm enforces a global budget $T$ via $t$ and a per-child test window $m_{\mathrm{exp}}$ via $s$.
\end{remark}

\subsection{Basic lemmas}
\label{app:basic-lemmas}

\begin{lemma}[Multi-sample falsification, exact-gap form; implies Lemma~\ref{lem:multi-sample}]
\label{app:lem:multi-sample}
Let $v$ be a node for which $Q_{C_v}$ is defined, and let $x_s\sim Q_{C_v}$ and $f_s\sim P_{x_s}$ for $m$ independent rounds. Then
\[
\Pr\!\left[\mathrm{Viol}(f_s;C_v)=0\ \text{for all }s=1,\ldots,m\right]=(1-q_v)^m\le\exp\!\left(-m c_{\mathrm{fal}}\Delta(v)^\beta\right).
\]
In particular, a node with $\Delta(v)>\varepsilon$ survives $m$ tests with probability at most $\exp(-m c_{\mathrm{fal}}\varepsilon^\beta)$, which is at most $\delta'$ when $m\ge\log(1/\delta')/(c_{\mathrm{fal}}\varepsilon^\beta)$.
\end{lemma}

\begin{proof}
Since $c_{\mathrm{fal}}\le1$ and $\Delta(v)\le1$, we have $c_{\mathrm{fal}}\Delta(v)^\beta\le1$, so Assumption~\ref{app:ass:feedback} gives $q_v\ge c_{\mathrm{fal}}\Delta(v)^\beta$, also when $\Delta(v)=0$. The rounds are independent by Assumption~\ref{app:ass:generation}, and $1-y\le e^{-y}$ gives the bound.
\end{proof}

\begin{lemma}[Safe commitments]
\label{app:lem:safety}
Suppose the refutational-feedback condition of Assumption~\ref{app:ass:feedback} and Assumptions~\ref{app:ass:generation} and~\ref{app:ass:eps-sep} hold. Run Algorithm~\ref{alg:robust} with any history-dependent proposer, test length $m$, at most $N$ tests, and $T=Nm$. Then each test has a complete window of $m$ rounds, and the probability that some child with $\Delta>0$ is committed is at most $N\exp(-m c_{\mathrm{fal}}\varepsilon^\beta)$.
\end{lemma}

\begin{proof}
The $n$-th test starts after at most $(n-1)m$ rounds, so it has $m$ rounds available, and a child is committed only after $m$ violation-free rounds. Before the first commitment of a child with $\Delta>0$, every committed node is viable, so by Assumption~\ref{app:ass:eps-sep} the child tested at that moment has $\Delta>\varepsilon$. Conditional on the history and the proposed child, the test rounds are fresh independent samples, so Lemma~\ref{app:lem:multi-sample} bounds its acceptance probability by $\exp(-m c_{\mathrm{fal}}\varepsilon^\beta)$. A union bound over the at most $N$ tests at which this first commitment can occur gives the claim.
\end{proof}

\subsection{Idealized sound-proposer regime}
\label{app:strict-regime}

We first analyze an idealized regime in which proposed constraints are always true. This result is mainly a warm-up: it shows that once proposals are sound, constraint-tree depth alone is enough to reach the latent feasible region.

\begin{theorem}[Coverage under a sound proposer]
\label{app:thm:strict}
Suppose the refutational-feedback condition of Assumption~\ref{app:ass:feedback} holds, together with Assumptions~\ref{app:ass:generation} and~\ref{app:ass:sound-proposer}. Consider the idealized variant of \textsc{TRACE} that tests the current node for $m_{\mathrm{exp}}\ge 1$ rounds and then calls $\mathrm{Propose}$ whenever the node survives those rounds without falsification. If the budget satisfies $T\ge K m_{\mathrm{exp}}$, then with probability one the procedure reaches a node $v$ with $C_{v}=C^*$. Consequently, if $\hat{x}\sim Q_{C_{v}}$, then $\hat{x}\in\mathrm{Region}(C^*)$ and $U(\hat{x})\ge U^*-\mathrm{Spread}(C^*)$.
\end{theorem}

\begin{proof}
We argue by induction over the committed path. Initially, the root satisfies $C_{v_0}=\emptyset\subseteq C^*$. Suppose the current node $v$ satisfies $C_v\subseteq C^*$. Since $x^*\in\mathrm{Region}(C_v)$, we have $U(v)=U^*$ and therefore $\Delta(v)=0$. By the no-false-alarm part of Assumption~\ref{app:ass:feedback}, the node is never falsified: $q_v=0$. Thus the node survives any test window of length $m_{\mathrm{exp}}$.

If $C_v\subsetneq C^*$, then by Assumption~\ref{app:ass:sound-proposer} the proposer returns a constraint $c\in C^*\setminus C_v$. The algorithm moves to the child $u$ with $C_u=C_v\cup\{c\}$. This child also satisfies $C_u\subseteq C^*$ and has depth one larger than $v$. Therefore, each successful expand step strictly increases $|C_v|$ while preserving the invariant $C_v\subseteq C^*$.

After at most $K$ such expansions, the committed node satisfies $C_v=C^*$. Each expansion consumes exactly $m_{\mathrm{exp}}$ test rounds, so the budget $T\ge K m_{\mathrm{exp}}$ suffices. Finally, by the support property of $Q_{C_v}$, $\hat{x}\sim Q_{C_v}$ implies $\hat{x}\in\mathrm{Region}(C^*)$. By the definition of $\mathrm{Spread}(C^*)$, $U(\hat{x})\ge U^*-\mathrm{Spread}(C^*)$.
\end{proof}

\subsection{Noisy-proposer regime}
\label{app:noisy-regime}

We now analyze the robust test-before-commit algorithm under a noisy proposer. This is the main theoretical setting used in the paper.

\begin{theorem}[Coverage in the falsification setting; full version of Theorem~\ref{thm:main-noisy}]
\label{app:thm:noisy}
Suppose the refutational-feedback condition of Assumption~\ref{app:ass:feedback} holds (its extraction clauses (a)--(b) are not needed), together with Assumptions~\ref{app:ass:generation}, \ref{app:ass:noisy-proposer}, and \ref{app:ass:eps-sep}. Fix $\delta\in(0,1)$ and $\varepsilon\in(0,1]$. If
\[
N_{\max}
:=
\left\lceil
H\log\frac{2\max(K,1)}{\delta}
\right\rceil,
\qquad
m_{\mathrm{exp}}
:=
\left\lceil
\frac{\log(2N_{\max}/\delta)}
{c_{\mathrm{fal}}\varepsilon^\beta}
\right\rceil,
\qquad
T:=N_{\max}m_{\mathrm{exp}},
\]
then with probability at least $1-\delta$ the deepest committed node $v$ satisfies $\mathrm{Region}(C_{v})\subseteq\mathrm{Region}(C^*)$. In the canonical finite-class proof, this region inclusion is obtained through the sufficient event $C_{v}\supseteq C^*$; the committed set may also include additional zero-gap constraints from $\mathcal{H}\setminus C^*$. Hence, if $\hat{x}\sim Q_{C_{v}}$, then $\hat{x}\in\mathrm{Region}(C^*)$ and $U(\hat{x})\ge U^*-\mathrm{Spread}(C^*)$. With $L_F:=\log(2K/\delta)$, the total interaction budget is
\[
T
=
\mathcal{O}\!\left(
\frac{
HL_F\log(4HL_F/\delta)
}{
c_{\mathrm{fal}}\varepsilon^\beta
}
\right).
\]
\end{theorem}

\begin{proof}
Write $N:=N_{\max}$ and $m:=m_{\mathrm{exp}}$, and let $\mathcal{B}$ be the event that a child with $\Delta>0$ is committed. By Lemma~\ref{app:lem:safety} and the choice of $m$, $\Pr(\mathcal{B})\le N\exp(-m c_{\mathrm{fal}}\varepsilon^\beta)\le\delta/2$.

Index the proposal calls, including those whose candidate gives an empty region, and let $\mathcal{F}_j$ be the history after $j$ calls and their tests. Fix $c\in C^*$, and let $A_{j,c}$ be the event that no child with $\Delta>0$ has been committed during the first $j$ calls and $c$ is still uncommitted after them. On $A_{j,c}$ the committed node $v$ is viable, so its region contains a global optimum, which satisfies $c$; hence the child $v\cup\{c\}$ is nonempty with $\Delta=0$, and $c\in\mathcal{H}_v$ by Remark~\ref{app:rem:lazy-filter}. For $j<N$, at most $j$ tests and $jm$ rounds have been used, so the next call takes place and any test it triggers has a complete window. This call proposes $c$ with conditional probability $1/|\mathcal{H}_v|\ge1/H$, and the resulting child then passes and is committed. Hence
\[
\Pr(A_{j+1,c}\mid\mathcal{F}_j)\le\Bigl(1-\frac{1}{H}\Bigr)\mathbf{1}_{A_{j,c}},
\qquad\text{so}\qquad
\Pr(A_{N,c})\le\Bigl(1-\frac{1}{H}\Bigr)^{N}\le e^{-N/H}.
\]
If $\mathcal{B}$ does not occur but some $c\in C^*$ is never committed, the run makes at least $N$ calls, since a viable run that misses a true constraint cannot stop earlier, and therefore $A_{N,c}$ occurs. By the choice of $N$,
\[
\Pr(C^*\nsubseteq C_v)\le\Pr(\mathcal{B})+\sum_{c\in C^*}\Pr(A_{N,c})\le\frac{\delta}{2}+Ke^{-N/H}\le\delta.
\]
On the complement, $C_v\supseteq C^*$, so $\mathrm{Region}(C_v)\subseteq\mathrm{Region}(C^*)$; the support of $Q_{C_v}$ gives $\hat{x}\in\mathrm{Region}(C^*)$, and the definition of $\mathrm{Spread}(C^*)$ gives $U(\hat{x})\ge U^*-\mathrm{Spread}(C^*)$. Each of the at most $N$ tests uses at most $m$ rounds, so $T=Nm$ suffices.
\end{proof}

\begin{corollary}[Expected simple regret]
\label{app:cor:expected-regret}
Under the setting of Theorem~\ref{app:thm:noisy}, the expected simple regret satisfies
\[
\mathbb{E}[r_T]
\le
\mathrm{Spread}(C^*)+\delta.
\]
\end{corollary}

\begin{proof}
Let $E$ be the success event in Theorem~\ref{app:thm:noisy}. Then $\Pr(E)\ge 1-\delta$. On $E$, the output satisfies $r_T=U^*-U(\hat{x})\le \mathrm{Spread}(C^*)$. On $E^c$, since $U\in[0,1]$, we have $r_T\le 1$. Therefore
\[
\mathbb{E}[r_T]
\le
(1-\delta)\mathrm{Spread}(C^*)+\delta\cdot 1
\le
\mathrm{Spread}(C^*)+\delta.
\]
\end{proof}

\subsection{Identification setting}
\label{app:identity-regime}

This section analyzes \textsc{TRACE-I}, which augments the falsification-setting test-before-commit core with an identification-aware proposer that biases its candidates toward constraints extracted from past feedback.

\paragraph{Identification-aware proposer.}
The proposer maintains a running set $W_t\subseteq\mathcal{H}$ initialized to $\emptyset$. After every round in which $\mathrm{Extract}(f_t)=c\neq\emptyset$, it updates $W_{t+1}\leftarrow W_t\cup\{c\}$, matching Line~10 of Algorithm~\ref{alg:robust}; under the soundness clause of Assumption~\ref{app:ass:feedback}, every such $c$ lies in $C^*$, so the update is sound regardless of whether $\mathrm{Viol}(f_t;C)=1$ also fires, and the discovery argument below only uses the subset of these updates that occur on rounds with $\mathrm{Viol}(f_t;C)=1$. At each proposal call from node $v$, it returns a uniform random element of $W_t\cap\mathcal{H}_v$ when this set is nonempty, and a uniform random element of $\mathcal{H}_v$ otherwise, with $\mathcal{H}_v=\mathcal{H}\setminus(C_v\cup R_v)$ as in Algorithm~\ref{alg:robust}; it returns $\emptyset$ if $\mathcal{H}_v=\emptyset$.

\begin{theorem}[Coverage in the identification setting; full version of Theorem~\ref{thm:identity}]
\label{app:thm:identity}
Suppose Assumptions~\ref{app:ass:feedback}, \ref{app:ass:generation}, and~\ref{app:ass:identity-progress} hold (in particular $C^*\subseteq\mathcal{H}$; Assumption~\ref{app:ass:identity-progress} implies Assumption~\ref{app:ass:eps-sep}). Fix $\delta\in(0,1)$ and $\varepsilon\in(0,1]$. With
\[
N_{\max}^{\mathrm{id}}
:= \left\lceil
\frac{64}{p_{\mathrm{ext}}}
\left(K+\log\frac{4}{\delta}\right)
\right\rceil,
\qquad
m_{\mathrm{exp}}^{\mathrm{id}}
:= \left\lceil
\frac{\log(2N_{\max}^{\mathrm{id}}/\delta)}
{c_{\mathrm{fal}}\varepsilon^\beta}
\right\rceil,
\qquad
T:=N_{\max}^{\mathrm{id}}m_{\mathrm{exp}}^{\mathrm{id}},
\]
\textsc{TRACE-Identification} enjoys an $(\varepsilon,\delta)$-coverage guarantee: with probability at least $1-\delta$, the deepest committed node $v$ satisfies $\mathrm{Region}(C_{v})\subseteq\mathrm{Region}(C^*)$. In the canonical proof this follows from the sufficient event $C_{v}\supseteq C^*$. Consequently, if $\hat{x}\sim Q_{C_{v}}$, then $\hat{x}\in\mathrm{Region}(C^*)$ and $U(\hat{x})\ge U^*-\mathrm{Spread}(C^*)$. With $L_I:=K+\log(4/\delta)$, the total interaction budget is
\[
T
=
\mathcal{O}\!\left(
\frac{
L_I\log(4L_I/(p_{\mathrm{ext}}\delta))
}{
p_{\mathrm{ext}}\,c_{\mathrm{fal}}\varepsilon^\beta
}
\right)
=
\widetilde{\mathcal{O}}\!\left(
\frac{K}{p_{\mathrm{ext}}\,c_{\mathrm{fal}}\varepsilon^\beta}
\right).
\]
\end{theorem}

\begin{proof}
Let $N:=N_{\max}^{\mathrm{id}}$, $m:=m_{\mathrm{exp}}^{\mathrm{id}}$, and $p:=p_{\mathrm{ext}}$. Assumption~\ref{app:ass:identity-progress} implies Assumption~\ref{app:ass:eps-sep}, so by Lemma~\ref{app:lem:safety} and the choice of $m$, with probability at least $1-\delta/2$ no child with $\Delta>0$ is committed during the first $N$ tests, and each of these tests has a complete window. Call \emph{recovery} the first time at which $C^*\subseteq C_v$.

Before recovery and before any child with $\Delta>0$ is committed, every committed node is viable and, by Assumption~\ref{app:ass:identity-progress}, every committed constraint is true, so $C_v\subseteq C^*$; soundness gives $W_t\subseteq C^*$. We track the potential
\[
\Psi:=|C_v|+|C_v\cup W_t|\le 2K,
\]
which never decreases and gains one unit when a true constraint is discovered and one unit when it is committed. Consider a test in this period. A prioritized test proposes a constraint of $W_t\cap\mathcal{H}_v\subseteq C^*\setminus C_v$, which is committed, so $\Psi$ increases. Otherwise $W_t\cap\mathcal{H}_v=\emptyset$, which at a viable node means $W_t\subseteq C_v$, because missing true constraints never enter $R_v$ (Remark~\ref{app:rem:lazy-filter}). A fallback proposal in $C^*$ is then committed and increases $\Psi$ by two. A fallback proposal outside $C^*$ is either committed, which is a commitment with $\Delta>0$, or rejected after a violation; in the latter case, with conditional probability at least $p$, the extractor returns a true constraint that is new relative to the working set at the start of the test, so $\Psi$ increases during this test (see the discussion after Assumption~\ref{app:ass:identity-progress}). Hence, given the history, each such test increases $\Psi$ or commits a child with $\Delta>0$ with probability at least $p$.

For $n\le N$, let $Y_n:=1$ if the $n$-th test increases $\Psi$ or commits a child with $\Delta>0$, and let $Y_n:=1$ if this test starts after recovery or after such a commitment, or does not take place. Before these events the run does not stop within $N$ tests: a missing true constraint stays in $\mathcal{H}_v$, and the budget covers $N$ complete windows. Therefore $\Pr(Y_n=1\mid\mathcal{F}_{n-1})\ge p$ for the history $\mathcal{F}_{n-1}$ before the $n$-th test, and the tower property gives
\[
\mathbb{E}\Big[\exp\Big(-\lambda\sum_{n=1}^{N}Y_n\Big)\Big]\le\big(1-p+pe^{-\lambda}\big)^{N}\qquad\text{for every }\lambda>0.
\]
Since $Np\ge 64(K+\log(4/\delta))$, we have $2K\le Np/2$ and $e^{-Np/8}\le\delta/2$, so the multiplicative Chernoff lower-tail bound gives $\Pr[\sum_{n=1}^{N}Y_n<2K]\le\delta/2$.

If recovery does not occur within the first $N$ tests and no child with $\Delta>0$ is committed, then every $Y_n=1$ records an increase of $\Psi$, which starts at $0$ and stays at most $2K-1$ because $|C_v|\le K-1$; hence $\sum_{n=1}^{N}Y_n<2K$. Therefore recovery occurs within the first $N$ tests, and thus within $T=Nm$ rounds, with probability at least $1-\delta$. Later commitments only add constraints, so at termination $C_v\supseteq C^*$ and $\mathrm{Region}(C_v)\subseteq\mathrm{Region}(C^*)$; the support of $Q_{C_v}$ and the definition of $\mathrm{Spread}(C^*)$ give the remaining claims, as in Corollary~\ref{app:cor:expected-regret}.
\end{proof}

\subsection{Anytime gap-UCB bound}
\label{app:ucb-bonus}

The survival argument also yields an anytime confidence bound on the gap of a tested node. It is not needed for Theorems~\ref{app:thm:noisy} and~\ref{app:thm:identity}, but it reads the test of Algorithm~\ref{alg:robust} as a confidence rule.

\begin{lemma}[Anytime gap upper confidence bound]
\label{app:lem:anytime-ucb}
Fix a node $v$ for which $Q_{C_v}$ is defined, and consider an independent sequence of rounds at $v$. For $m\ge 1$, let $A_m:=\{\mathrm{Viol}(f_s;C_v)=0\ \text{for all }1\le s\le m\}$ be the event that the first $m$ rounds at $v$ produce no violation signal. For $\delta\in(0,1)$, define
\[
\widehat{\Delta}_{\mathrm{UCB}}(v;m,\delta)
:=
\left(\frac{\log(1/\delta)}{c_{\mathrm{fal}}m}\right)^{1/\beta}.
\]
Then, with probability at least $1-\delta$, $\Delta(v)\le \widehat{\Delta}_{\mathrm{UCB}}(v;m,\delta)$ holds simultaneously for every $m\ge 1$ such that $A_m$ occurs. The events $A_m$ concern the initial violation-free prefix; violation-free runs that restart after a violation are not covered, since each restart would require additional confidence.
\end{lemma}

\begin{proof}
If $\Delta(v)=0$ there is nothing to prove. Otherwise let $a:=c_{\mathrm{fal}}\Delta(v)^\beta$. Lemma~\ref{app:lem:multi-sample} gives $\Pr(A_m)\le e^{-ma}$. The radius decreases in $m$, and $\Delta(v)>\widehat{\Delta}_{\mathrm{UCB}}(v;m,\delta)$ holds exactly when $m\ge m_*:=\lfloor\log(1/\delta)/a\rfloor+1$. Because $A_{m+1}\subseteq A_m$, the failure event $\{\exists m\ge1: A_m\text{ occurs and }\Delta(v)>\widehat{\Delta}_{\mathrm{UCB}}(v;m,\delta)\}$ equals $A_{m_*}$, whose probability is at most $e^{-m_*a}<\delta$.
\end{proof}

\paragraph{Interpretation.}
The bound in Lemma~\ref{app:lem:anytime-ucb} is the gap-dependent analogue of a UCB confidence radius. Because the prefix events are nested, no union bound over $m$ is needed, and for fixed $\delta$ the radius decays as $\mathcal{O}(m^{-1/\beta})$ without a logarithmic factor in $m$; larger $\beta$ corresponds to a more ambiguous feedback channel and slower decay.

\begin{remark}[Anytime reading of the test in Algorithm~\ref{alg:robust}]
\label{app:rem:adaptive}
Accepting a child once $\widehat{\Delta}_{\mathrm{UCB}}(v;m,\delta)\le\varepsilon$ and rejecting it at the first violation is exactly the test of Algorithm~\ref{alg:robust}: the radius depends on the data only through survival, so acceptance happens after the fixed window $\lceil\log(1/\delta)/(c_{\mathrm{fal}}\varepsilon^\beta)\rceil$, which equals $m_{\mathrm{exp}}$ when $\delta$ is replaced by $\delta/(2N_{\max})$. Lemma~\ref{app:lem:anytime-ucb} therefore certifies, at every round of a test, how large the gap of a surviving child can still be.
\end{remark}

\section{Additional Experimental Details}
\label{app:experimental-details}

This appendix provides the full experimental details omitted from the main text. Section~\ref{app:env-details} describes the environments and implementation choices. Section~\ref{app:saturated-tasks} reports the three remaining task instances of the main comparison and Wilson confidence intervals for the main success-rate tables. Section~\ref{app:recmovie-full} gives the RecMovie ablations, robustness checks, and equal-output-budget comparisons. Section~\ref{app:twentyq-eda-full} covers \textsc{20Q}, \textsc{EDA}, \textsc{Mastermind}, and Poem. Section~\ref{app:trace-recmovie} gives an example execution trace, and Section~\ref{app:prompts} lists the prompts of the LLM-based components.

\subsection{Environment and implementation details}
\label{app:env-details}

\paragraph{Common setup.}
Unless otherwise noted, \textsc{TRACE} uses $m_{\mathrm{exp}}=15$ and $N_{\max}=15$. The headline $\mathrm{Viol}$ and $\mathrm{Extract}$ modules are rule-based: a regex over each environment's feedback template, with task-specific decoders that map matched phrases to canonical predicates in $\mathcal{H}$. \textsc{TRACE-Falsification} (\textsc{TRACE-F}) uses a uniform proposer; \textsc{TRACE-Identification} (\textsc{TRACE-I}) uses the same decoder to extract identities and prioritize them in the proposer. Qwen3-4B served locally is used for (i) the prompt-based LLM baselines (\textsc{DirectLLM}, \textsc{Reflexion}, \textsc{Self-Refine}), (ii) the Qwen3-4B configuration of \textsc{TravelPlanner}'s constrained plan generator $Q_C$ in Section~\ref{sec:travelplanner}, and (iii) the LLM-based $\mathrm{Viol}$ / $\mathrm{Extract}$ ablations in Tables~\ref{tab:app-proposer-recmovie} and~\ref{tab:app-detector-recmovie}; these ablations confirm that headline TRACE rows do not rely on LLM oracle modules, since substituting the LLM detector for the rule-based detector drops TRACE-F's interaction-budget any-success from $57\%$ to $21\%$ on RecMovie. The \textsc{Mastermind} column of the main comparison (Table~\ref{tab:app-saturated-tasks}) fixes $K=3$ and $n_{\mathrm{colors}}=6$; Haiku targets the $5$--$7$--$5$ syllable pattern and Tanka the $5$--$7$--$5$--$7$--$7$ pattern. Wilson $95\%$ intervals for all main-table cells are reported in Section~\ref{app:main-ci}; small differences within overlapping intervals should not be over-interpreted. Prompt-based baselines (\textsc{DirectLLM}, \textsc{Reflexion}, \textsc{Self-Refine}) are operational comparisons under a fixed local model, rather than a multi-model robustness study. All reported results use $100$ seeds per cell, except \textsc{TravelPlanner}, which uses $50$ validation queries (Section~\ref{sec:travelplanner}). The primary metric is \emph{env.success}, the fraction of seeds on which a method produces a satisfying output within its interaction budget; for RecMovie \textsc{TRACE} rows this is the final-output success of the algorithm's converged recommendation list. Section~\ref{app:recmovie-full} additionally reports an interaction-budget any-success diagnostic and a soft-utility breakdown. When available, we also report $\mathrm{superset}(C^*)$, which checks whether the committed constraint set contains all true constraints and serves as a sufficient syntactic proxy for region inclusion.

\paragraph{Compute resources.}
All experiments were run on a local workstation with a single NVIDIA GeForce RTX 5080 GPU. LLM-based baselines, the \textsc{TravelPlanner} plan generator $Q_C$, and the LLM-based $\mathrm{Viol}$ / $\mathrm{Extract}$ ablation rows used Qwen3-4B on this workstation, and the DeepSeek-V4-Flash configuration of \textsc{TravelPlanner} called the DeepSeek API; headline TRACE runs and the remaining environment simulation, rule-based decoding, and tabulation code ran on CPU.

\paragraph{LLM baselines.}
LLM baselines use Qwen3-4B. We compare against \textsc{DirectLLM}, \textsc{Reflexion}~\citep{shinn2023reflexion}, and \textsc{Self-Refine}~\citep{madaan2023selfrefine}. \textsc{DirectLLM} conditions on the full interaction history. \textsc{Reflexion} follows the original trial-based design with sliding-window reflective memory; \textsc{Self-Refine} follows \citet{madaan2023selfrefine} with full-history critique-and-revise.

\paragraph{RecMovie.}
Each seed samples a user profile across multiple attribute dimensions (movie type, year range, genre, age restriction); the corresponding latent constraint set $C^*$ collects the dimensions actually set in the profile, so $K = |C^*|$ varies per seed. On each round, the agent proposes three movie titles, and the environment returns template-based verbal feedback. All LLM baselines run for $T=20$ total interaction steps on RecMovie; trial-based Reflexion partitions this budget across trials. Across $100$ seeds, the LLM baselines that succeed do so within roughly $3$ steps on average (DirectLLM mean $3.7$, Self-Refine mean $2.9$), while failed seeds exhaust nearly the full budget. Appendix~\ref{app:equal-budget} compares all methods under common caps of $20$ and $60$ outputs; there, the windowed baseline variants gain $2$--$10$pp from the larger cap.

\paragraph{\textsc{20Q}.}
\textsc{20Q} is a controlled synthetic environment with five categorical attribute dimensions yielding $|\mathcal{H}|=17$ candidate predicates over a dense $|\mathcal{X}|=432$-entity catalog (every attribute combination has exactly one entity). Each episode samples a profile inducing $C^*$ whose size $K = |C^*|$ varies per seed; feedback is templated per dimension, e.g., ``\textit{the target is not a mammal}''. Because the catalog is dense, $\mathrm{Region}(C_v)$ is nonempty for every consistent constraint set $C_v$.

\paragraph{\textsc{EDA Things}.}
We use the $99$-entity test set of Apple's Entity-Deduction Arena (\textsc{EDA Things})~\citep{zhang2024probing}, annotated by us with five categorical attribute dimensions; annotations were initialized by Qwen3-4B and manually corrected. Each episode samples a profile inducing $C^*$ whose size $K = |C^*|$ varies per seed. The catalog is sparse, so when a committed constraint set yields an empty region, the generator falls back to the entity satisfying the largest number of committed predicates.

\paragraph{\textsc{Mastermind}.}
We use a per-position \textsc{Mastermind} environment, following the \emph{Clear Mastermind} feedback style~\citep{li2024playing}. The hidden code has length $K$, and each position takes one color from a palette of size $n_{\mathrm{colors}}$. Feedback is per-position: each proposed position is marked as ``exact'', ``present'', or ``absent''. The candidate class is $\mathcal{H}=\{(i,c): i\in[K],\, c\in\mathrm{palette}\}$ with $|\mathcal{H}|=K\cdot n_{\mathrm{colors}}$. This environment allows controlled sweeps over both $K$ and $|\mathcal{H}|$.

\paragraph{LLF-Bench Poem.}
For LLF-Bench Poem~\citep{cheng2023llf}, we use the Haiku and Tanka tasks without changing their reward, feedback content, or target pattern. Haiku has target syllable pattern $5$--$7$--$5$ and thus $K=3$ line-level constraints. Tanka has target pattern $5$--$7$--$5$--$7$--$7$ and thus $K=5$ constraints. We apply a runtime patch to the upstream \texttt{cmudict} lookup so that missing words fall back to the existing heuristic syllable estimator; this patch is used only for words outside the CMU dictionary. The generator is a word-bucket generator that samples words from syllable-indexed buckets and composes each line to satisfy a proposed syllable count. The predicate class is $\mathcal{H}=\{(\mathrm{line}_i,\mathrm{syll}_s): i\in[K],\ s\in\{3,4,5,6,7,8\}\}$ with $|\mathcal{H}|=6K$. \textsc{TRACE} reads only feedback, not the task prompt, so the target syllable pattern is discovered from the environment replies.

\paragraph{$K$ across environments.}
$K = |C^*|$ is fixed per experiment in \textsc{Mastermind} (main comparison: $K=3$, $n_{\mathrm{colors}}=6$; H-sweep: $K=3$; K-sweep: $K\in\{2,4,6,8,10\}$) and Poem (Haiku: $K=3$; Tanka: $K=5$). In RecMovie, \textsc{20Q}, and \textsc{EDA Things}, $K$ is sampled per seed by the environment's profile sampler.

\subsection{Remaining main-comparison tasks and confidence intervals}
\label{app:saturated-tasks}

\paragraph{Remaining task instances.}
For headline focus, Table~\ref{tab:main-comparison} reports only the three tasks where the falsification-to-identification gap is non-saturated and the lossy-extraction story (RecMovie) is empirically separable. Table~\ref{tab:app-saturated-tasks} records the corresponding natural-feedback main-comparison numbers on the three remaining task instances: \textsc{Mastermind} ($K=3$, $n_{\mathrm{colors}}=6$), LLF-Bench Haiku ($K=3$ syllable constraints), and LLF-Bench Tanka ($K=5$). \textsc{TRACE-I} reaches $99$--$100\%$ on all three; natural-feedback prompt baselines remain below $30\%$. Wilson $95\%$ intervals are in Table~\ref{tab:app-main-ci} (Mmind / Haiku / Tanka columns).

\begin{table}[h]
\centering
\small
\caption{Main-comparison env.success on the three task instances moved out of Table~\ref{tab:main-comparison} for headline focus. Entries are env.success within each method's interaction budget; higher is better.}
\label{tab:app-saturated-tasks}
\begin{tabular}{lrrr}
\toprule
\textbf{Method} & \textbf{Mastermind} & \textbf{Haiku} & \textbf{Tanka} \\
\midrule
Random                                 & $\phantom{1}4\%$  & $\phantom{1}7\%$  & $\phantom{1}0\%$ \\
DirectLLM                              & $15\%$ & $20\%$ & $\phantom{1}4\%$ \\
Reflexion~\citep{shinn2023reflexion}     & $\phantom{1}9\%$  & $\phantom{1}2\%$  & $\phantom{1}1\%$ \\
Self-Refine~\citep{madaan2023selfrefine} & $29\%$ & $14\%$ & $\phantom{1}0\%$ \\
\textsc{TRACE-F}                       & $54\%$ & $54\%$ & $14\%$ \\
\textbf{\textsc{TRACE-I}}              & $\mathbf{99\%}$ & $\mathbf{100\%}$ & $\mathbf{100\%}$ \\
\bottomrule
\end{tabular}
\end{table}

\paragraph{Wilson confidence intervals.}\label{app:main-ci}
All success rates in the main tables use $100$ seeds per cell, except Table~\ref{tab:travelplanner} ($50$ queries). Tables~\ref{tab:app-main-ci}--\ref{tab:app-same-oracle-ci} report Wilson $95\%$ confidence intervals rounded to the nearest percentage point. These intervals are intended for interpreting success-rate differences, not for changing the primary point estimates shown in the main text.

\begin{table}[h]
\centering
\small
\setlength{\tabcolsep}{3pt}
\caption{Wilson $95\%$ confidence intervals for natural-feedback baselines and \textsc{TRACE} rows. The first three columns correspond to the Natural columns (for baselines) and the \emph{Falsif.}/\emph{Identif.}\ columns (for \textsc{TRACE}) of Table~\ref{tab:main-comparison}; the other three columns correspond to the columns of Table~\ref{tab:app-saturated-tasks}. Entries are intervals for env.success percentages.}
\label{tab:app-main-ci}
\begin{tabular}{lrrrrrr}
\toprule
\textbf{Method} & \textbf{Rec.} & \textbf{\textsc{20Q}} & \textbf{\textsc{EDA}} & \textbf{Mmind} & \textbf{Haiku} & \textbf{Tanka} \\
\midrule
Random & $[2,11]$ & $[32,51]$ & $[65,82]$ & $[2,10]$ & $[3,14]$ & $[0,4]$ \\
DirectLLM & $[11,26]$ & $[7,20]$ & $[34,53]$ & $[9,23]$ & $[13,29]$ & $[2,10]$ \\
Reflexion & $[18,35]$ & $[6,19]$ & $[45,64]$ & $[5,16]$ & $[1,7]$ & $[0,5]$ \\
Self-Refine & $[14,30]$ & $[64,81]$ & $[69,85]$ & $[21,39]$ & $[9,22]$ & $[0,4]$ \\
\textsc{TRACE-F} & $[41,61]$ & $[76,90]$ & $[79,92]$ & $[44,63]$ & $[44,63]$ & $[9,22]$ \\
\textbf{\textsc{TRACE-I}} & $\mathbf{[83,94]}$ & $\mathbf{[96,100]}$ & $\mathbf{[95,100]}$ & $\mathbf{[95,100]}$ & $\mathbf{[96,100]}$ & $\mathbf{[96,100]}$ \\
\bottomrule
\end{tabular}
\end{table}

\begin{table}[h]
\centering
\small
\caption{Wilson $95\%$ confidence intervals for the \textsc{Mastermind} fixed-$K$ $H$-sweep in Table~\ref{tab:h-scaling}. Entries are intervals for env.success percentages.}
\label{tab:app-h-ci}
\begin{tabular}{rrrr}
\toprule
\textbf{$n_{\mathrm{colors}}$} & \textbf{$|\mathcal{H}|$} & \textbf{\textsc{TRACE-F} CI} & \textbf{\textsc{TRACE-I} CI} \\
\midrule
$3$  &  $9$ & $[96,100]$ & $[96,100]$ \\
$6$  & $18$ & $[96,100]$ & $[96,100]$ \\
$10$ & $30$ & $[93,99]$ & $[96,100]$ \\
$15$ & $45$ & $[80,93]$ & $[96,100]$ \\
$20$ & $60$ & $[65,82]$ & $[95,100]$ \\
\bottomrule
\end{tabular}
\end{table}

\begin{table}[h]
\centering
\small
\caption{Wilson $95\%$ confidence intervals for the same-decoder structured comparison in Table~\ref{tab:structure-necessity}. Entries are intervals for env.success percentages.}
\label{tab:app-structure-ci}
\begin{tabular}{lrrr}
\toprule
\textbf{Configuration} & \textbf{RecMovie} & \textbf{\textsc{20Q}} & \textbf{\textsc{EDA Things}} \\
\midrule
\textbf{\textsc{TRACE-I}} & $\mathbf{[83,94]}$ & $\mathbf{[96,100]}$ & $[95,100]$ \\
Direct-Commit (best variant) & $[16,32]$ & $\mathbf{[96,100]}$ & $\mathbf{[96,100]}$ \\
\bottomrule
\end{tabular}
\end{table}

\begin{table}[h]
\centering
\small
\caption{Wilson $95\%$ confidence intervals for the binary and identity oracle baselines in the \emph{Falsif.}\ and \emph{Identif.}\ columns of Table~\ref{tab:main-comparison}. Entries are intervals for env.success percentages.}
\label{tab:app-same-oracle-ci}
\begin{tabular}{lrrr}
\toprule
\textbf{Method} & \textbf{Rec.} & \textbf{\textsc{20Q}} & \textbf{\textsc{EDA}} \\
\midrule
\multicolumn{4}{l}{\textit{Group A: binary oracle}} \\
DirectLLM     & $[19,36]$ & $[34,53]$  & $[76,90]$ \\
Reflexion     & $[26,45]$ & $[21,39]$  & $[59,77]$ \\
Self-Refine   & $[18,35]$ & $[13,28]$  & $[31,50]$ \\
\textbf{\textsc{TRACE-F}}   & $\mathbf{[41,61]}$ & $\mathbf{[76,90]}$  & $\mathbf{[79,92]}$ \\
\midrule
\multicolumn{4}{l}{\textit{Group B: identity oracle}} \\
DirectLLM     & $[31,50]$ & $\mathbf{[96,100]}$ & $\mathbf{[96,100]}$ \\
Reflexion     & $[33,52]$ & $\mathbf{[96,100]}$ & $\mathbf{[96,100]}$ \\
Self-Refine   & $[24,42]$ & $\mathbf{[96,100]}$ & $\mathbf{[96,100]}$ \\
\textbf{\textsc{TRACE-I}} & $\mathbf{[83,94]}$ & $\mathbf{[96,100]}$ & $[95,100]$ \\
\bottomrule
\end{tabular}
\end{table}

\subsection{Full RecMovie results}
\label{app:recmovie-full}

\paragraph{Soft-utility breakdown.}
Table~\ref{tab:app-recmovie-soft} extends the main RecMovie comparison with a soft-utility breakdown of the final recommendation list. The soft utility $U_{\mathrm{final}}$ of a list is the mean, over its three recommended titles, of the fraction of canonical $C^*$ predicates satisfied by each title; soft regret is $1-U_{\mathrm{final}}$. The \textbf{final env.} column reproduces the headline final-output success used in the main tables.

\begin{table}[h]
\centering
\small
\caption{RecMovie final-output soft-score / soft-regret over $100$ seeds. $U_{\mathrm{final}}$ is mean partial-credit utility and $1-U_{\mathrm{final}}$ is soft regret. The \textbf{final env.} column matches the headline RecMovie numbers in the main tables.}
\label{tab:app-recmovie-soft}
\begin{tabular}{lrrr}
\toprule
\textbf{Method} & \textbf{final env.} & \textbf{$U_{\mathrm{final}}$} & \textbf{soft regret} \\
\midrule
\textsc{TRACE-F}   & $51\%$ & $0.641$ & $0.359$ \\
\textbf{\textsc{TRACE-I}} & $\mathbf{90\%}$ & $\mathbf{0.831}$ & $\mathbf{0.169}$ \\
Direct-AND            & $15\%$ & $0.496$ & $0.504$ \\
Direct-OR             & $23\%$ & $0.528$ & $0.472$ \\
Direct-RANKED         & $22\%$ & $0.527$ & $0.473$ \\
Random                & $\phantom{1}5\%$ & $0.328$ & $0.672$ \\
\bottomrule
\end{tabular}
\end{table}

\paragraph{Interaction-budget any-success.}
For completeness, the alternative interaction-budget any-success diagnostic (the fraction of seeds on which any tested recommendation list during the run was accepted by the env, regardless of whether the algorithm subsequently moved past it) gives \textsc{TRACE-F}: $57\%$ and \textsc{TRACE-I}: $95\%$ on the same $100$ seeds; LLM baselines and direct-commit variants return on first success, so for them this any-success rate coincides with the final-output rate above.

\paragraph{Proposer ablation.}
Table~\ref{tab:app-proposer-recmovie} fixes the \textsc{TRACE} commit/reject scaffold and varies how candidate identities are proposed. The uniform proposer realizes the falsification setting; Incremental-LM extracts predicate candidates from each feedback message with an LLM; the decoded proposer uses the task's canonical feedback templates.

\begin{table}[h]
\centering
\small
\caption{Proposer variants inside \textsc{TRACE} on RecMovie with the same rule-based detector; entries are interaction-budget any-success.}
\label{tab:app-proposer-recmovie}
\begin{tabular}{lr}
\toprule
\textbf{Proposer} & \textbf{any-success} \\
\midrule
\textsc{TRACE-F} (uniform)
& 57\% \\
Incremental-LM (atomic per-feedback)
& 74\% \\
\textbf{\textsc{TRACE-I} (decoded)}
& \textbf{95\%} \\
\bottomrule
\end{tabular}
\end{table}

\paragraph{Detector ablation.}
Table~\ref{tab:app-detector-recmovie} compares detector implementations on RecMovie while fixing the algorithm to \textsc{TRACE-F}. The rule-based detector is stronger in this semi-templated setting.

\begin{table}[h]
\centering
\small
\caption{Detector variants inside \textsc{TRACE-F} on RecMovie. \emph{Overall}: any-success. \emph{Final Output}: final-output env.success (matches Table~\ref{tab:main-comparison}).}
\label{tab:app-detector-recmovie}
\begin{tabular}{lrr}
\toprule
\textbf{Detector} & \textbf{Overall} & \textbf{Final Output} \\
\midrule
\textbf{Rule-based detector}
& \textbf{57\%} & \textbf{51\%} \\
LLM, zero-shot
& 21\% & 12\% \\
LLM, chain-of-thought
& 21\% & \phantom{1}8\% \\
\bottomrule
\end{tabular}
\end{table}

\paragraph{Direct-commit variants.}
The headline same-decoder comparison is reported in Table~\ref{tab:structure-necessity} of the main text. Table~\ref{tab:app-structure-recmovie} separates the three RecMovie direct-commit variants: \emph{AND} over all wanted predicates, \emph{OR} over at least one predicate per dimension, and a \emph{RANKED} top-$k$ match score.

\begin{table}[h]
\centering
\small
\caption{Same decoder, with versus without the \textsc{TRACE} test-before-commit scaffold, on RecMovie. Entries are final-output env.success over $100$ seeds; for direct-commit baselines, final-output success coincides with any-success since they return on first hit.}
\label{tab:app-structure-recmovie}
\begin{tabular}{lr}
\toprule
\textbf{Configuration} & \textbf{env.success} \\
\midrule
\textbf{\textsc{TRACE-I}}
& \textbf{90\%} \\
Direct-AND (strict)
& 15\% \\
Direct-OR (per-dim disjunction)
& 23\% \\
Direct-RANKED (top-$k$ match score)
& 22\% \\
\bottomrule
\end{tabular}
\end{table}

\paragraph{Feedback omission.}
We test an i.i.d.\ feedback-dropout model in which each feedback message is replaced by the empty string with probability $p$. Table~\ref{tab:app-noise-recmovie} reports RecMovie results.

\begin{table}[h]
\centering
\small
\caption{Interaction-budget any-success of \textsc{TRACE-I} on RecMovie under i.i.d.\ feedback omission, $100$ seeds per cell. Performance is flat through $20\%$ omission and degrades gradually afterwards.}
\label{tab:app-noise-recmovie}
\begin{tabular}{lr}
\toprule
\textbf{Omission rate} $p$ & \textbf{any-success} \\
\midrule
$0\%$  & 95\% \\
$10\%$ & 95\% \\
$20\%$ & 95\% \\
$30\%$ & 94\% \\
$50\%$ & 93\% \\
$70\%$ & 87\% \\
\bottomrule
\end{tabular}
\end{table}

\paragraph{False identities.}
Table~\ref{tab:app-false-identity} isolates false extracted identities from feedback omission. For every decoded feedback predicate, with probability $p_{\mathrm{false}}$ we replace it by a uniformly sampled wrong predicate from the same attribute dimension that is not in the seed's canonical $C^*$. The \textsc{TRACE-I} variant uses these corrupted identities only in the proposer while keeping the rule-based detector clean; direct baselines accumulate the same corrupted identities without test-before-commit. To capture degradation across the entire run rather than just the converged output, this sweep uses the interaction-budget any-success diagnostic; at $p_{\mathrm{false}}=0$ it therefore reports the $95\%$ any-success value rather than the $90\%$ final-output headline. As false identities increase, \textsc{TRACE-I} degrades gradually ($95\%\to67\%$ any-success from $p_{\mathrm{false}}=0$ to $0.5$), whereas all direct-commit variants fall to $8\%$. The direct methods can still have high $\mathrm{superset}(C^*)$ because they accumulate many constraints, including the true ones; the any-success column reveals the cost of unfiltered false positives, while \textsc{TRACE-I}'s test-before-commit filters them out before commitment.

\begin{table}[h]
\centering
\small
\setlength{\tabcolsep}{3pt}
\caption{False-identity-only RecMovie sweep over $100$ seeds. Each extracted predicate is independently replaced by a wrong same-dimension predicate with probability $p_{\mathrm{false}}$. Success is interaction-budget any-success, unlike the final-output rates in Table~\ref{tab:false-identity-main}.}
\label{tab:app-false-identity}
\begin{tabular}{llrr}
\toprule
\textbf{$p_{\mathrm{false}}$} & \textbf{Method} & \textbf{any-succ.} & \textbf{$\mathrm{superset}(C^*)$} \\
\midrule
$0.0$ & \textbf{\textsc{TRACE-I}} & $\mathbf{95\%}$ & $59\%$ \\
$0.0$ & Direct-AND                    & $15\%$ & $94\%$ \\
$0.0$ & Direct-OR                     & $23\%$ & $94\%$ \\
$0.0$ & Direct-RANKED                 & $22\%$ & $94\%$ \\
\midrule
$0.1$ & \textbf{\textsc{TRACE-I}} & $\mathbf{90\%}$ & $54\%$ \\
$0.1$ & Direct-AND                    & $12\%$ & $86\%$ \\
$0.1$ & Direct-OR                     & $12\%$ & $93\%$ \\
$0.1$ & Direct-RANKED                 & $19\%$ & $94\%$ \\
\midrule
$0.2$ & \textbf{\textsc{TRACE-I}} & $\mathbf{83\%}$ & $49\%$ \\
$0.2$ & Direct-AND                    & $\phantom{0}8\%$ & $78\%$ \\
$0.2$ & Direct-OR                     & $12\%$ & $92\%$ \\
$0.2$ & Direct-RANKED                 & $16\%$ & $93\%$ \\
\midrule
$0.3$ & \textbf{\textsc{TRACE-I}} & $\mathbf{85\%}$ & $53\%$ \\
$0.3$ & Direct-AND                    & $\phantom{0}8\%$ & $73\%$ \\
$0.3$ & Direct-OR                     & $\phantom{0}8\%$ & $90\%$ \\
$0.3$ & Direct-RANKED                 & $12\%$ & $92\%$ \\
\midrule
$0.5$ & \textbf{\textsc{TRACE-I}} & $\mathbf{67\%}$ & $38\%$ \\
$0.5$ & Direct-AND                    & $\phantom{0}7\%$ & $58\%$ \\
$0.5$ & Direct-OR                     & $\phantom{0}7\%$ & $90\%$ \\
$0.5$ & Direct-RANKED                 & $\phantom{0}8\%$ & $93\%$ \\
\bottomrule
\end{tabular}
\end{table}

\paragraph{Equal output budgets.}\label{app:equal-budget}
All runs use the same $100$ RecMovie seeds and report final-output success, i.e., whether the single output returned by a method satisfies the user. The common budget unit is a generated output that the environment evaluates. Under a cap $B$, \textsc{TRACE-I} spends at most $B-1$ outputs on feedback-bearing tests and returns one final output, sampled from its deepest committed node, within the cap; each seed and cap is run independently. Iterative baselines stop at their first accepted output and otherwise return their output at the cap. Best-of-$N$ draws $B$ independent generations and is granted an oracle selector that returns a satisfying sample whenever one exists. All baselines receive the environment's feedback mapped to the same canonical identity projection over the $36$ predicates used by the \textsc{TRACE-I} decoder. The baseline configurations of Table~\ref{tab:main-comparison} were run only with $T=20$. To evaluate baselines with $60$ outputs we use three variants of the same methods: DirectLLM-W8 and Self-Refine-W8 condition only on the eight most recent interactions, and Reflexion-NoCycle keeps four trials but disables the early termination of a trial that repeats the same recommendation. For these variants, success at $B=20$ is read from the same $60$-output runs as whether the first accepted output occurs within the first $20$ outputs. Tables~\ref{tab:app-equal-budget} and~\ref{tab:app-low-interaction} give the numbers plotted in Figure~\ref{fig:equal-budget-recmovie}, and Table~\ref{tab:app-equal-budget-cost} reports resource use.

\begin{table}[h]
\centering
\small
\caption{Final-output success on RecMovie under common output caps ($100$ seeds; Wilson $95\%$ intervals). The Table~\ref{tab:main-comparison} configurations were run only with $T=20$.}
\label{tab:app-equal-budget}
\begin{tabular}{lcc}
\toprule
\textbf{Method} & $B=20$ & $B=60$ \\
\midrule
\textsc{TRACE-I} ($m_{\mathrm{exp}}=5$) & $\mathbf{73\%}$ {\scriptsize [$64$, $81$]} & $\mathbf{86\%}$ {\scriptsize [$78$, $91$]} \\
DirectLLM (Table~\ref{tab:main-comparison} config.) & $40\%$ {\scriptsize [$31$, $50$]} & n/a \\
Reflexion (Table~\ref{tab:main-comparison} config.) & $42\%$ {\scriptsize [$33$, $52$]} & n/a \\
Self-Refine (Table~\ref{tab:main-comparison} config.) & $32\%$ {\scriptsize [$24$, $42$]} & n/a \\
DirectLLM-W8 & $38\%$ {\scriptsize [$29$, $48$]} & $48\%$ {\scriptsize [$38$, $58$]} \\
Reflexion-NoCycle & $35\%$ {\scriptsize [$26$, $45$]} & $37\%$ {\scriptsize [$28$, $47$]} \\
Self-Refine-W8 & $37\%$ {\scriptsize [$28$, $47$]} & $42\%$ {\scriptsize [$33$, $52$]} \\
Best-of-$N$ (oracle selector) & $17\%$ {\scriptsize [$11$, $26$]} & $21\%$ {\scriptsize [$14$, $30$]} \\
\bottomrule
\end{tabular}
\end{table}

\begin{table}[h]
\centering
\small
\caption{\textsc{TRACE-I} test-window sweep on RecMovie under common output caps (final-output success with Wilson $95\%$ intervals; feedback rounds per episode).}
\label{tab:app-low-interaction}
\begin{tabular}{rrcrr}
\toprule
$m_{\mathrm{exp}}$ & $B$ & \textbf{Success} & \textbf{Mean rounds} & \textbf{P90 rounds} \\
\midrule
$1$ & $20$ & $65\%$ {\scriptsize [$55$, $74$]} & $4.43$ & $5$ \\
$3$ & $20$ & $78\%$ {\scriptsize [$69$, $85$]} & $13.42$ & $19$ \\
$3$ & $60$ & $77\%$ {\scriptsize [$68$, $84$]} & $14.26$ & $22$ \\
$5$ & $20$ & $73\%$ {\scriptsize [$64$, $81$]} & $18.61$ & $19$ \\
$5$ & $60$ & $86\%$ {\scriptsize [$78$, $91$]} & $22.08$ & $27$ \\
\bottomrule
\end{tabular}
\end{table}

\begin{table}[h]
\centering
\small
\caption{Resource use per seed on RecMovie, averaged over each method's complete $100$-seed logs (baselines: $60$-output runs). \emph{Evals}: environment-evaluated outputs. \emph{Gen.} and \emph{Aux.}: LLM calls for candidate generation and for reflection or critique.}
\label{tab:app-equal-budget-cost}
\begin{tabular}{lrrrrr}
\toprule
\textbf{Method} & \textbf{Evals} & \textbf{Gen.\ LLM} & \textbf{Aux.\ LLM} & \textbf{LLM tokens} & \textbf{Wall (s)} \\
\midrule
\textsc{TRACE-I} ($m_{\mathrm{exp}}=5$) & $22.1$ & $0$ & $0$ & $0$ & $0.4$ \\
\textsc{TRACE-I} ($m_{\mathrm{exp}}=15$, default) & $60.6$ & $0$ & $0$ & $0$ & $0.9$ \\
DirectLLM-W8 & $36.3$ & $36.3$ & $0$ & $21{,}275$ & $14.4$ \\
Reflexion-NoCycle & $39.9$ & $39.9$ & $2.0$ & $25{,}845$ & $14.2$ \\
Self-Refine-W8 & $38.4$ & $38.4$ & $38.0$ & $62{,}133$ & $42.4$ \\
Best-of-$N$ ($N=60$) & $60.0$ & $60.0$ & $0$ & $13{,}479$ & $16.7$ \\
\bottomrule
\end{tabular}
\end{table}

\paragraph{Trusted initial constraints.}
Requirements that are already known, for example stated in the user's initial request or verified in an earlier session, can seed the root: \textsc{TRACE} then starts from $C_{v_0}=C_{\mathrm{init}}$, never re-tests these constraints, and searches only for the remaining ones; $C_{\mathrm{init}}=\emptyset$ recovers Algorithm~\ref{alg:robust}. If $C_{\mathrm{init}}\subseteq C^*$, the root is viable because $\mathcal{X}^*\subseteq\mathrm{Region}(C^*)\subseteq\mathrm{Region}(C_{\mathrm{init}})$, and the proofs of Theorems~\ref{thm:main-noisy} and~\ref{thm:identity} apply to the residual search, since they only require every committed node, including the root, to be viable. In Table~\ref{tab:app-trusted-root}, \textsc{TRACE-I} ($m_{\mathrm{exp}}=5$) receives $\min(2,|C^*|)$ oracle-provided true constraints at initialization (the four seeds with $|C^*|=1$ receive one). Because these constraints are supplied by the oracle, the gain is an upper bound on the value of verified context.

\begin{table}[h!]
\centering
\small
\caption{\textsc{TRACE-I} ($m_{\mathrm{exp}}=5$) on RecMovie with an oracle-trusted root (final-output success with Wilson $95\%$ intervals; mean feedback rounds per episode).}
\label{tab:app-trusted-root}
\begin{tabular}{lcccc}
\toprule
\textbf{Root} & $B=20$ & $B=60$ & \textbf{Rounds} ($B=20$) & \textbf{Rounds} ($B=60$) \\
\midrule
$C_{\mathrm{init}}=\emptyset$ & $73\%$ {\scriptsize [$64$, $81$]} & $86\%$ {\scriptsize [$78$, $91$]} & $18.61$ & $22.08$ \\
$|C_{\mathrm{init}}|\le 2$, oracle-trusted & $\mathbf{90\%}$ {\scriptsize [$83$, $94$]} & $\mathbf{91\%}$ {\scriptsize [$84$, $95$]} & $10.09$ & $10.77$ \\
\bottomrule
\end{tabular}
\end{table}

\subsection{Full \textsc{20Q}, \textsc{EDA}, \textsc{Mastermind}, and Poem results}
\label{app:twentyq-eda-full}

\paragraph{\textsc{20Q} and \textsc{EDA}: coverage versus env.success.}
Table~\ref{tab:app-20q-eda-superset} compares env.success and the $\mathrm{superset}(C^*)$ coverage diagnostic on \textsc{20Q} and \textsc{EDA}. The theorem directly controls coverage; env.success may also depend on catalog grounding and generator fallback behavior.

\begin{table}[h!]
\centering
\small
\caption{Coverage diagnostic versus env.success on \textsc{20Q} and \textsc{EDA}.
On sparse \textsc{EDA}, catalog grounding can decouple env.success from
coverage; the theorem-aligned diagnostic is $\mathrm{superset}(C^*)$.}
\label{tab:app-20q-eda-superset}
\begin{tabular}{lrrrr}
\toprule
& \multicolumn{2}{c}{\textsc{20Q}} & \multicolumn{2}{c}{\textsc{EDA Things}} \\
\cmidrule(lr){2-3} \cmidrule(lr){4-5}
\textbf{Configuration} & env.succ. & superset$(C^*)$ & env.succ. & superset$(C^*)$ \\
\midrule
\textsc{TRACE-F} (uniform)
& 84\% & 82\% & 87\% & 84\% \\
\textbf{\textsc{TRACE-I} (decoded)}
& \textbf{100\%} & \textbf{100\%} & \textbf{99\%} & \textbf{99\%} \\
\bottomrule
\end{tabular}
\end{table}

\paragraph{\textsc{Mastermind}: fixed-budget $K$-sweep.}\label{app:mastermind-full}
Table~\ref{tab:app-mastermind-k} varies $K$ at fixed $n_{\mathrm{colors}}=6$. Since $|\mathcal{H}|=6K$, this sweep varies $K$ and $H$ jointly. It is therefore not a direct test of the theorem's logarithmic $K$ factor; rather, it is a fixed-budget stress test showing how the uniform proposer degrades as the candidate class grows. The complementary $H$-sweep at fixed $K=3$ is reported in Table~\ref{tab:h-scaling} of the main text; on that sweep, $\mathrm{superset}(C^*)$ tracks env.success exactly except at the largest setting $n_{\mathrm{colors}}=20$, where \textsc{TRACE-F} reaches env.success $74\%$ but $\mathrm{superset}(C^*)=72\%$ (\textsc{TRACE-I} has both at $99\%$).

\begin{table}[h]
\centering
\small
\caption{Fixed-budget $K$-sweep on \textsc{Mastermind} with $n_{\mathrm{colors}}=6$ fixed, so $|\mathcal{H}|=6K$.}
\label{tab:app-mastermind-k}
\begin{tabular}{rrrrr}
\toprule
\textbf{K} & \textbf{\textsc{TRACE-F} succ.} & \textbf{Binary attempts} & \textbf{\textsc{TRACE-I} succ.} & \textbf{Identity attempts} \\
\midrule
 2 & $\mathbf{73\%}$  & $10.09$ & $99\%$  &  $5.14$ \\
 4 & $35\%$  & $13.84$ & $\mathbf{100\%}$ &  $5.56$ \\
 6 & $\phantom{1}4\%$  & $14.95$ & $\mathbf{100\%}$ &  $7.29$ \\
 8 & $\phantom{1}1\%$  & $15.00$ & $\mathbf{100\%}$ &  $9.11$ \\
10 & $\phantom{1}0\%$  & $15.00$ & $\mathbf{100\%}$ & $10.98$ \\
\bottomrule
\end{tabular}
\end{table}

\paragraph{Poem: open-vocabulary action generation.}\label{app:poem-full}
For LLF-Bench Poem, the action is free-form English text rather than a finite catalog item. We implement (i) a word-bucket generator that composes each line with a proposed syllable count by sampling CMU-dict words from syllable-indexed buckets; (ii) a rule-based detector that regex-parses native feedback of the form ``\emph{The line: X has $n$ syllables. It should only have $m$ syllables.}''; and (iii) \textsc{TRACE-F} and \textsc{TRACE-I} variants over line-syllable predicates.

Table~\ref{tab:app-poem} reports Haiku and Tanka results. \textsc{TRACE-I} covers $C^*$ in every seed on both forms. \textsc{TRACE-F} degrades from Haiku to Tanka as $K$ and $|\mathcal{H}|=6K$ increase under the fixed proposal budget.

\begin{table}[h]
\centering
\small
\caption{\textsc{TRACE} on LLF-Bench Poem with Haiku and Tanka forms. We use $100$ seeds per cell and $m_{\mathrm{exp}}=N_{\max}=15$. Brackets are Wilson $95\%$ confidence intervals. \textsc{TRACE-I} exploits per-line target syllable counts revealed in the feedback; \textsc{TRACE-F} must search over $|\mathcal{H}|=6K$ predicates.}
\label{tab:app-poem}
\begin{tabular}{lrrrr}
\toprule
& \multicolumn{2}{c}{\textbf{Haiku ($K=3$)}} & \multicolumn{2}{c}{\textbf{Tanka ($K=5$)}} \\
\cmidrule(lr){2-3} \cmidrule(lr){4-5}
\textbf{Method} & env.succ. & attempts & env.succ. & attempts \\
\midrule
\textsc{TRACE-F}
& $54\%$ \hfill {\scriptsize [$44$,$63$]} & $12.57$
& $14\%$ \hfill {\scriptsize [$8$,$22$]} & $14.77$ \\
\textsc{TRACE-I}
& $\mathbf{100\%}$ \hfill {\scriptsize [$96$,$100$]} & $3.80$
& $\mathbf{100\%}$ \hfill {\scriptsize [$96$,$100$]} & $5.81$ \\
\bottomrule
\end{tabular}
\end{table}

The Poem experiment shows that \textsc{TRACE} is not limited to catalog recommendation. With a suitable constrained generator, the same commit/reject scaffold applies to open-vocabulary text generation. The experiment is still structured: the feedback contains parseable line-level syllable information. Richer unstructured natural-language feedback, where rule-based extraction is unlikely to suffice, remains an important direction for future work.

\newpage

\subsection{Example execution trace on RecMovie}
\label{app:trace-recmovie}

To make the algorithm's per-round behavior concrete, the listing below shows the full execution trace of \textsc{TRACE-I} on RecMovie seed $16$ with the headline default $m_{\mathrm{exp}}=N_{\max}=15$ (rule-based detector + decoded proposer). The latent constraint set is $C^*=\{\text{type\_=TV show},\ \text{year=80s},\ \text{genre=Sport}\}$. The trace illustrates three mechanisms central to the framework:
\begin{itemize}[leftmargin=*]
\item \textbf{Test-before-commit catches a spurious child cheaply.} Attempt~$n{=}1$ proposes the spurious constraint \texttt{type\_=movie}; the very first feedback contradicts it (``not all TV shows'') and also reveals two missing constraints (``not exclusively Sport TV shows''; ``not from the 80s''), so $\mathrm{Viol}$ fires after one round. The child is rejected after consuming only $1$ round of the per-test cap $m_{\mathrm{exp}}=15$.
\item \textbf{Multi-sample falsification on viable children is silent.} Attempts $n=2$, $3$, $4$ each propose a true latent constraint; feedback is consistent with the tested set on all $15$ rounds (we elide the $14$ duplicate rounds), so each child survives the test window and is committed.
\item \textbf{Identification access guides the proposer.} The decoded proposer reads the feedback for extracted identities; by attempt~$n{=}2$ the working set $W$ already contains the $\text{type\_=TV show}$ identity extracted from the attempt~$1$ feedback, leading to a true child being proposed and committed. Subsequent commits ($\text{year=80s}$, $\text{genre=Sport}$) follow the same pattern.
\end{itemize}
After $4$ attempts and $46$ interaction rounds (well under the budget $T = N_{\max}\cdot m_{\mathrm{exp}} = 225$), the deepest committed node $v^{\dagger}$ satisfies $C_{v^{\dagger}} = C^*$ and the output sample $\hat{x}$ consists of three Sport TV shows from the 80s, satisfying every latent constraint.

\begin{tcolorbox}[breakable, colback=gray!5, colframe=gray!50, fontupper=\footnotesize, left=1mm, right=1mm,
  title={\textbf{Listing}: \textsc{TRACE-I} execution trace on RecMovie seed $16$},
  fonttitle=\bfseries\small,
  ]
\begin{verbatim}
======================================================================
  TRACE-I run on RecMovie  |  seed=16  m_exp=15  N_max=15
======================================================================
  C* (latent constraints) : {genre=Sport, type_=TV show, year=80s}
  |H| (candidate class)   : 36

----------------------------------------------------------------------
Attempt n=1   (parent C_v = {},  depth=0)
  Proposed:  type_=movie     [spurious]
  Child C_u = {type_=movie}
    s=1  feedback: "The recommendations are not all TV shows. The
         suggested titles aren't exclusively Sport TV shows. The
         recommendations are not from the 80s."   [Viol fired]
  -> REJECT   (rounds used: 1/15)

----------------------------------------------------------------------
Attempt n=2   (parent C_v = {},  depth=0)
  Proposed:  type_=TV show   [true latent]
  Child C_u = {type_=TV show}
    s=1  feedback: "You've recommended TV shows, which is awesome!
         Among the recommended titles, you'll find more than Sport
         TV shows. The recommendations are not from the 80s."
    ... (s=2..15: feedback unchanged, child survives)
  -> COMMIT   (rounds used: 15/15)
      C_v advances to depth 1: {type_=TV show}

----------------------------------------------------------------------
Attempt n=3   (parent C_v = {type_=TV show},  depth=1)
  Proposed:  year=80s        [true latent]
  Child C_u = {type_=TV show, year=80s}
    s=1  feedback: "You've recommended TV shows, which is awesome!
         Among the recommended titles, you'll find more than Sport
         TV shows. The recommendations are all from the 80s, great!"
    ... (s=2..15: feedback unchanged, child survives)
  -> COMMIT   (rounds used: 15/15)
      C_v advances to depth 2: {type_=TV show, year=80s}

----------------------------------------------------------------------
Attempt n=4   (parent C_v = {type_=TV show, year=80s},  depth=2)
  Proposed:  genre=Sport     [true latent]
  Child C_u = {genre=Sport, type_=TV show, year=80s}
    s=1  feedback: "You've recommended TV shows, which is awesome!
         The recommendations are all Sport, nice! The recommendations
         are all from the 80s, great!"
    ... (s=2..15: feedback unchanged, child survives)
  -> COMMIT   (rounds used: 15/15)
      C_v advances to depth 3: {genre=Sport, type_=TV show, year=80s}

----------------------------------------------------------------------
Final v_dagger: C = {genre=Sport, type_=TV show, year=80s}
Output x_hat:   "Sport i 2'eren", "ZDF Sport Extra", "A Question of 
                Sport"
Termination:    empty_region_cap (no more children to test)
Total rounds:   46 / 225        Total attempts:  4 / 15
Coverage (C* subset C_{v_dagger}):  True
======================================================================
\end{verbatim}
\end{tcolorbox}

\subsection{Prompts for LLM-based components}
\label{app:prompts}

For reproducibility we record the verbatim prompts used by the LLM-based ablations in Appendix~\ref{app:recmovie-full} and by the \textsc{TravelPlanner} runs in Section~\ref{sec:travelplanner}. The prompts were written for Qwen3-4B served locally; the DeepSeek-V4-Flash configuration of Section~\ref{sec:travelplanner} sends the same \textsc{TravelPlanner} prompts through the DeepSeek API with thinking mode enabled; the trailing ``/no\_think'' marker disables Qwen3's chain-of-thought block when not needed.

\begin{tcolorbox}[breakable, colback=gray!5, colframe=gray!50, fontupper=\footnotesize, left=1mm, right=1mm,
  title={\textbf{LMViolationDetector: zero-shot} (Table~\ref{tab:app-detector-recmovie}, $21\%$)},
  fonttitle=\bfseries\small]
\begin{verbatim}
[system] You are a careful logical judge. Answer strictly as
         instructed. /no_think

[user] A movie recommender has committed to the following
       assumptions about the user's preferences:
         - <predicate 1>
         - <predicate 2>
         ...

       The user responded with this feedback after seeing the
       latest recommendations:
       """<feedback>"""

       QUESTION: Does the feedback CONTRADICT any of the
       committed assumptions?
       - Answer YES if the feedback indicates that at least one
         committed assumption is wrong (e.g. the feedback says
         the recommendations were not of a type/era/genre the
         assumption claimed).
       - Answer NO if the feedback is consistent with the
         assumptions, or only complains about a dimension not
         in the assumptions.

       Answer with a single word: YES or NO. /no_think
\end{verbatim}
\end{tcolorbox}

\begin{tcolorbox}[breakable, colback=gray!5, colframe=gray!50, fontupper=\footnotesize, left=1mm, right=1mm,
  title={\textbf{LMViolationDetector: chain-of-thought} (Table~\ref{tab:app-detector-recmovie}, $21\%$)},
  fonttitle=\bfseries\small]
\begin{verbatim}
[system + user same as zero-shot, replacing the final line:]

       First, in ONE short sentence, explain your reasoning.
       Then on a new line, write ANSWER: YES or ANSWER: NO.
       /no_think
\end{verbatim}
\end{tcolorbox}

\begin{tcolorbox}[breakable, colback=gray!5, colframe=gray!50, fontupper=\footnotesize, left=1mm, right=1mm,
  title={\textbf{IncrementalLMProposer: per-feedback extraction} (Table~\ref{tab:app-proposer-recmovie}, $74\%$)},
  fonttitle=\bfseries\small]
\begin{verbatim}
[system] You extract structured preferences from feedback.
         Follow the output format precisely. /no_think

[user] You are analysing ONE feedback message from a user to a
       movie recommender.

       CRITICAL CONVENTION (RecMovie): feedback of the form
         "The recommendations are not X" / "not all X" /
         "not from X"
       is a CORRECTION meaning the user WANTS X on that
       dimension. Positive phrasing ("... wonderful!" /
       "... great!") also confirms the user wants that value.
       Phrase "past few years" refers to the year bucket
       "recent".

       Dimensions you can emit (use EXACTLY these names):
         - type_  (values: movie, TV show)
         - year   (values: 50s, 60s, 70s, 80s, 90s, 2000s,
                   recent)
         - genre  (values: Action, Comedy, Drama, Thriller,
                   Horror, Sci-Fi, ...)

       Feedback:
       """<feedback>"""

       List the user's preferences revealed by this feedback.
       Output ONE line per preference in the format
       ``dim=value``. If no preference is revealed, output
       nothing. No other text, no prose. /no_think
\end{verbatim}
\end{tcolorbox}

\begin{tcolorbox}[breakable, colback=gray!5, colframe=gray!50, fontupper=\footnotesize, left=1mm, right=1mm,
  title={\textbf{TravelPlanner plan generator: system prompt} (\texttt{tp\_v6\_no\_cost\_prompt}, used in Section~\ref{sec:travelplanner})},
  fonttitle=\bfseries\small]
\begin{verbatim}
You are a meticulous travel-planner that produces machine-
readable JSON plans.

You will receive:
  - A travel query (origin, destination, days, budget).
  - Optional committed constraints the user requires.
  - Reference information listing the only valid entities
    (flights, restaurants, attractions, accommodations) you
    may use.

Output rules (every plan MUST satisfy ALL of these):

 1. JSON-only output: a single object {"plan": [...]} with
    one entry per day. No prose, no markdown.

 2. Entity-format rule: for breakfast / lunch / dinner /
    attraction / accommodation fields, write "Name, City"
    (Name copied verbatim from reference). Multiple
    attractions per day separated by semicolons.

 3. Travel-day current_city format: "from CityA to CityB" on
    travel days; the city name only on stationary days.

 4. Transportation paradigm: pick ONE paradigm for the trip
    (Flight or Self-driving) and stick to it. Stationary days
    MUST use transportation = "-".

 5. Closed-circle trip: day 1 starts at the origin city; the
    final day returns to the origin.

 6. Sandbox-only entities: every Name MUST appear verbatim in
    the reference; on any "from/to" pattern only use city
    names from the reference.

 7. Diversity & accommodation: do not repeat restaurants or
    attractions; use the same accommodation across consecutive
    nights; respect the reference's minimum-nights column.

 8. Completeness: fill every field; use "-" only when the
    field genuinely does not apply.

 9. Day-1 accommodation: never write "-".

10. Budget is a meaningful target, not an upper bound: prefer
    balanced choices when budget allows; never exceed the
    budget; if spending under 50% of budget, reconsider.
\end{verbatim}
\end{tcolorbox}

\noindent The user prompt is rendered by \texttt{build\_user\_prompt} with fields \texttt{Query / Origin / Destination / Days / People / Budget}, the canonicalized \texttt{Committed} block (predicates already accepted by \textsc{TRACE}), the \texttt{Reference Information} (optionally pre-filtered by the committed constraints), and the closing instruction ``\texttt{Respond with only the JSON object: \{"plan": [...]\}}''. Plans are evaluated by the env's built-in commonsense and hard-constraint evaluators; we do not implement an additional commonsense prompt of our own.

\end{document}